\documentclass{article}
\PassOptionsToPackage{numbers}{natbib}
\usepackage[preprint]{neurips_2024}
\usepackage[utf8]{inputenc}
\usepackage[T1]{fontenc}
\usepackage{url}
\usepackage{booktabs}
\usepackage{amsfonts}
\usepackage{nicefrac}
\usepackage{microtype}
\usepackage[svgnames]{xcolor}
\usepackage{lmodern}
\usepackage[hidelinks,colorlinks,allcolors=DarkBlue]{hyperref}
\usepackage[ruled]{algorithm2e}

\usepackage{tikz}
\usepackage{subcaption}
\usepackage{graphicx}
\usepackage{array,multirow}
\usepackage{caption} 
\usepackage{amsthm,amsmath,amssymb}
\usepackage{mathtools}
\usepackage{thm-restate}
\usepackage[nameinlink]{cleveref}
\usepackage{xspace}
\usepackage{xfrac}
\usepackage{float}

\SetCommentSty{mycommfont}

\renewcommand{\paragraph}[1]{\noindent\textbf{#1}~}

\newtheorem{theorem}{Theorem}

\newtheorem{proposition}{Proposition}

\theoremstyle{definition}
\newtheorem{definition}{Definition}

\newtheorem{openq}{Open Question }
\theoremstyle{remark}

\usepackage{tcolorbox}
\renewenvironment{openq}{\refstepcounter{openq}\begin{tcolorbox}[colback=gray!7, colframe=black, boxsep=1pt, left=4pt, right=4pt, rounded corners]\textbf{Open Question \theopenq:}}{\end{tcolorbox}}

\renewcommand{\le}{\leqslant}
\renewcommand{\leq}{\leqslant}
\renewcommand{\ge}{\geqslant}
\renewcommand{\geq}{\geqslant}

\DeclarePairedDelimiter{\set}{\{}{\}}

\newcommand{\bbR}{\mathbb{R}}
\newcommand{\bbN}{\mathbb{N}}

\newcommand{\e}{\varepsilon}

\DeclareMathOperator{\argmin}{\arg\min}

\DeclarePairedDelimiter{\ceil}{\lceil}{\rceil}

\newcommand{\gc}{\textsc{GreedyCapture}\xspace}
\newcommand{\gcsub}{\textsc{SmallestAgentBall}\xspace}
\newcommand{\gcc}{\textsc{GreedyCohesiveClustering}\xspace}
\newcommand{\gccsub}{\textsc{Most Cohesive Cluster}\xspace}
\newcommand{\auditfjr}{\textsc{AuditFJR}\xspace}

\newcommand{\alg}{\mathcal{A}}

\newcommand{\sd}{\textsc{SmallestDiameter}\xspace}

\title{Proportional Fairness in\\Non-Centroid Clustering\thanks{A preliminary version will be published in the proceedings of the 38th Annual Conference on Neural Information Processing Systems (NeurIPS), 2024.}}

\author{%
  Ioannis Caragiannis\\
  Aarhus University\\
  \texttt{iannis@cs.au.dk}\\
  \And
  Evi Micha\\
  Harvard University\\
  \texttt{emicha@seas.harvard.edu}\\
  \And
  Nisarg Shah\\
  University of Toronto\\
  \texttt{nisarg@cs.toronto.edu}
}

\begin{document}

\maketitle

\begin{abstract}
    We revisit the recently developed framework of proportionally fair clustering, where the goal is to provide group fairness guarantees that become stronger for groups of data points (agents) that are large and cohesive. Prior work applies this framework to centroid clustering, where the loss of an agent is its distance to the centroid assigned to its cluster. We expand the framework to \emph{non-centroid clustering}, where the loss of an agent is a function of the other agents in its cluster, by adapting two proportional fairness criteria --- the core and its relaxation, fully justified representation (FJR) --- to this setting. 
    
    We show that the core can be approximated only under structured loss functions, and even then, the best approximation we are able to establish, using an adaptation of the \gc algorithm developed for centroid clustering~\cite{chen2019proportionally,micha2020proportionally}, is unappealing for a natural loss function. In contrast, we design a new (inefficient) algorithm, \gcc, which achieves the relaxation FJR exactly under arbitrary loss functions, and show that the efficient \gc algorithm achieves a constant approximation of FJR. We also design an efficient auditing algorithm, which estimates the FJR approximation of any given clustering solution up to a constant factor. Our experiments on real data suggest that traditional clustering algorithms are highly unfair, whereas \gc is considerably fairer and incurs only a modest loss in common clustering objectives. 
\end{abstract}

\section{Introduction}\label{sec:intro}
Clustering is a fundamental task in unsupervised learning, where the goal is to partition a set of $n$ points into $k$ clusters $C = (C_1,\ldots,C_k)$ in such a way that points within the same cluster are close to each other (measured by a distance function $d$) and points in different clusters are far from each other. This goal is materialized through a variety of objective functions, the most popular of which is the $k$-means objective: $\sum_{i=1}^k \frac{1}{|C_i|} \cdot \sum_{x,y \in C_i} d(x,y)^2$.

When the points are in a Euclidean space, the $k$-means objective can be rewritten as $\sum_{i=1}^k \sum_{x \in C_i} d(x,\mu_i)^2$, where $\mu_i = \frac{1}{|C_i|} \sum_{x \in C_i} x$ is the mean (also called the \emph{centroid}) of cluster $C_i$.\footnote{Centroids can be defined in non-Euclidean spaces, e.g., as $\mu_i = \argmin_y \frac{1}{|C_i|} \sum_{x \in C_i} d(x,y)^2$.} This gives rise to centroid clustering, where deciding where to place the $k$ cluster centers is viewed as the task and the clusters are implicitly formed when each point is assigned to its nearest cluster center. 

In the literature on fairness in centroid clustering, the loss of each data point (hereinafter, agent) is defined as the distance to the nearest cluster center~\cite{chen2019proportionally}; here, the cluster centers do not merely help rewrite the objective, but play an essential role. This is a reasonable model for applications such as facility location, where the loss of an agent indeed depends on how much they have to travel to get to the nearest facility. 

But in other applications of clustering, we simply partition the agents and agents prefer to be close to other agents in their cluster---there are no ``cluster centers'' that they prefer to be close to. For example, in clustered federated learning~\cite{sattler2020clustered}, the goal is to cluster the agents and have agents in each cluster collaboratively learn a model; naturally, agents would want other agents in their cluster to have similar data distributions, so the model learned is accurate on their own data distribution.\footnote{Prior work formulates this as centroid clustering~\cite{ray2022fairness,ray2024fair}, where the principal also chooses a model for each cluster, but this goes against the federated learning setting.} Other examples where we want to cluster nearby points together without defining cluster centers include document clustering~\cite{janani2019text}, image segmentation for biomedical applications~\cite{huang2019brain}, and social network segmentation~\cite{lee2019social}. 

While there exist plenty of clustering objectives which do not require defining cluster centers (such as the first formulation of the $k$-means objective above), in order to reason about fairness we need to define the loss of each agent under a non-centroid clustering and explore the tradeoff between the losses of different agents. We initiate the study of proportional fairness in non-centroid clustering. 

We follow the idea of proportional fairness outlined in a recent line of work~\cite{chen2019proportionally,micha2020proportionally,kellerhals2023proportional,aziz2023proportionally}, which ensures that no group of at least $\sfrac{n}{k}$ agents should ``improve'' (formalized later) by forming a cluster of its own.\footnote{Groups with fewer than $\sfrac{n}{k}$ agents are not deemed to be entitled to form a cluster.} Our main research questions are:
\begin{quote}
    \emph{Can we obtain compelling proportional fairness guarantees for non-centroid clustering as with centroid clustering? Do the algorithms known to work well for centroid clustering also work well for non-centroid clustering? Can we audit the proportional fairness of a given algorithm?}
\end{quote}

\subsection{Our Contributions}\label{sec:contrib}
In non-centroid clustering, we are given a set $N$ of $n$ points (agents) and the desired number of clusters $k$. The goal is to partition the agents into (at most) $k$ clusters $C = (C_1,\ldots,C_k)$. Each agent $i$ has a loss function $\ell_i$, and her loss under clustering $C$ is $\ell_i(C(i))$, where $C(i)$ denotes the cluster containing her. We study both the general case where the loss functions of the agents can be arbitrary, and structured cases where the loss of an agent for a cluster is the average or maximum of her distances --- according to a given distance metric --- to the agents in the cluster. In the latter case, our theoretical results hold for general metric spaces, as they rely solely on the satisfaction of the triangle inequality.

We study two proportional fairness guarantees, formally defined in \Cref{sec:model}: \emph{the core}~\cite{Gil53} and its relaxation, \emph{fully justified representation} (FJR)~\cite{peters2021proportional}. Both have been studied for centroid clustering~\cite{chen2019proportionally,micha2020proportionally,aziz2023proportionally}, but we are the first to study them in non-centroid clustering.  

A summary of a selection of our results is presented in \Cref{tab:summary-pf}, with the cell values indicating approximation ratios (lower is better, $1$ is optimal). 

\begin{table}[htb!]
\centering
\def\arraystretch{1.3}
\setlength\tabcolsep{9pt}
\begin{tabular}{|c||c|c|c|}
\hline
\textbf{Loss Functions} & \textbf{Core UB} & \textbf{Core LB} & \textbf{FJR}\\\hline\hline
Arbitrary & \multicolumn{2}{c|}{$\infty$} & 1\\\hline
Average & $O(\sfrac{n}{k})$ (polytime) & $1.3$ & $1$ ($4$ in polytime)\\\hline
Maximum & $2$ (polytime) & 1 & $1$ ($2$ in polytime)\\\hline
\end{tabular}
\caption{The feasible core and FJR approximation guarantees, both existentially and in polynomial time. In each case, we can obtain a better FJR approximation than the core approximation.}
\label{tab:summary-pf}
\end{table}

Our results show the promise of FJR: while it is a slight  relaxation of the core, it is satisfiable even under arbitrary loss functions, whereas the core can be unsatisfiable even under more structured loss functions. The existential result for FJR is achieved using a simple (but inefficient) algorithm we design, \gcc, which is an adaptation of the Greedy Cohesive Rule from social choice theory~\cite{peters2021proportional}. The core approximations as well as efficient FJR approximations are achieved using an efficient version of it, which turns out to be an adaptation of the \gc algorithm that has been introduced for centroid clustering~\cite{chen2019proportionally,micha2020proportionally}. We show that the FJR approximation achieved by \gc stems from the fact that its key subroutine achieves a constant approximation of that of the \gcc algorithm. 

Next, we turn to auditing the FJR approximation of a given clustering. Surprisingly, we show that the same technique that we use to algorithmically \emph{achieve a constant approximation of FJR} can be used to also \emph{estimate the FJR approximation of any given clustering}, up to a constant factor ($4$ for the average loss and $2$ for the maximum loss). 

We compare \gc to popular clustering methods, $k$-means++ and $k$-medoids, on three real datasets. We observe that in terms of both average and maximum loss, \gc provides significantly better approximations to both FJR and the core, and this fairness advantage comes at only a modest cost in terms of traditional clustering objectives, including those that $k$-means++ and $k$-medoids are designed to optimize.

\subsection{Related Work}\label{sec:related}
In recent years, there has been an active line of research related to fairness in clustering~\cite{chhabra2021overview}. With a few exceptions, most of the work focuses on centroid-based clustering, where each agent cares about their distance from the closest cluster center. Mostly related to ours is the work by~\citet{chen2019proportionally}, who introduced the idea of proportionality through the core in centroid clustering. Their work has been revisited by~\citet{micha2020proportionally} for specific metric spaces. More recently, \citet{aziz2023proportionally} also introduced the relaxation of the core, fully justified representation, in centroid-based clustering. While one of our main algorithms, \gc, is a natural adaptation of the main algorithm used in all these works, there are significant differences between the two settings.

First, in centroid-based clustering, \gc provides a constant approximation to the core\cite{chen2019proportionally}, while in the non-centroid case this approximation is not better than $O(n/k)$ for the average loss function. Second, in centroid-based clustering, \gc returns a solution that satisfies  FJR exactly\cite{aziz2023proportionally}. Here, for the non-centroid case, even though we know that an exact FJR solution always exists, \gc is shown to just provide an approximation better than $4$ for the average loss and $2$ for the maximum loss. In more specific metric spaces, \citet{micha2020proportionally} show that a solution in the core always exists in the line. Here, we demonstrate that while this remains true for the maximum loss, it is not the case for the average loss, where the core can be empty. Finally, \citet{chen2019proportionally} conducted experiments using real data in which $k$-means++ performs better than \gc. However, for the same datasets, we found that \gc significantly outperforms $k$-means++ in the non-centroid setting. 

Fairness in non-centroid clustering has received significantly less attention. \citet{ahmadi2022individual} recently introduced a notion of individual stability which indicates that no agent should prefer another cluster over the one they have been assigned to.  \citet{micha2020proportionally}
studied the core when the goal is to create a balanced clustering (i.e. all clusters have almost equal size) and the agents have positive utilities for other agents. More generally, the hedonic games literature (e.g., see \cite{AS16} for an early survey on the topic and \cite{ABB+19} for a recent model that is close to the current paper) is also relevant to non-centroid clustering as it examines coalition formation. While the core concept has been extensively studied in hedonic games, there are two main differences with our work. First, subsets of any size can deviate to form their own cluster, rather than only proportionally eligible ones, and second, no approximate guarantees to the core have been provided, to the best of our knowledge. 

\section{Model}\label{sec:model}
For $t \in \bbN$, let $[t] \triangleq \set{1,\ldots,t}$. We are given a set $N$ of $n$ agents, and the desired number of clusters $k$. Each agent $i \in N$ has an associated \emph{loss function} $\ell_i : 2^N \setminus 2^{N \setminus \set{i}} \to \bbR_{\ge 0}$, where $\ell_i(S)$ is the cost to agent $i$ for being part of group $S$. A $k$-clustering\footnote{We simply call it clustering when the value of $k$ is clear from the context.} $C = (C_1,\ldots,C_k)$ is a partition of $N$ into $k$ clusters,\footnote{Technically, we have up to $k$ clusters as $C_t$ is allowed to be empty for any $t \in [k]$.} where $C_t \cap C_{t'} = \emptyset$ for $t \neq t'$ and $\cup_{t=1}^k C_t = N$. With slight abuse of notation, denote by $C(i)$ the cluster that contains agent $i$. Then, the loss of agent $i$ under this clustering is $\ell_i(C(i))$. 

\paragraph{Loss functions.} We study three classes of loss functions; for each class, we seek fairness guarantees that hold for any loss functions the agents may have from that class. A distance metric over $N$ is given by $d : N \times N \to \bbR_{\ge 0}$, which satisfies: (i) $d(i,i)=0$ for all $i \in N$, (ii) $d(i,j) = d(j,i)$ for all $i,j \in N$, and (iii) $d(i,j) \le d(i,k)+d(k,j)$ for all $i,j,k \in N$ (triangle inequality). 

\begin{itemize}
    \item \emph{Arbitrary losses.} In this most general class, the loss $\ell_i(S)$ can be an arbitrary non-negative number for each agent $i \in N$ and cluster $S \ni i$. 
    \item \emph{Average loss.} Here, we are given a distance metric $d$ over $N$, and $\ell_i(S) = \frac{1}{|S|} \sum_{j \in S} d(i,j)$ for each agent $i \in N$ and cluster $S \ni i$. Informally, agent $i$ prefers the agents in her cluster to be close to her on average.
    \item \emph{Maximum loss.} Again, we are given a distance metric $d$ over $N$, and $\ell_i(S) = \max_{j \in S} d(i,j)$ for each agent $i \in N$ and cluster $S \ni i$. Informally, agent $i$ prefers that no agent in her cluster to be too far from her.
\end{itemize}

\section{Core}\label{sec:core}
Perhaps the most widely recognized proportional fairness guarantee is \emph{the core}. Informally, an outcome is in the core if no group of agents $S \subseteq N$ can choose another (partial) outcome that (i) they are entitled to choose based on their proportion of the whole population ($\sfrac{|S|}{|N|}$), and (ii) makes every member of group $S$ happier. The core was proposed and widely studied in the resource allocation literature from microeconomics~\cite{Gil53,Var74,CFSW19}, and it has been adapted recently to centroid clustering~\cite{chen2019proportionally,micha2020proportionally}. When forming $k$ clusters out of $n$ agents, a group of agents $S$ is deemed worthy of forming a cluster of its own if and only if $|S| \ge \sfrac{n}{k}$. In centroid clustering, such a group can choose any location for its cluster center. In the following adaptation to non-centroid clustering, no such consideration is required. 

\begin{definition}[$\alpha$-Core]
    For $\alpha \ge 1$, a $k$-clustering $C=(C_1,\ldots,C_k)$ is said to be in the $\alpha$-core if there is no group of agents $S \subseteq N$ with $|S| \ge \sfrac{n}{k}$ such that $\alpha \cdot  \ell_i(S) < \ell_i(C(i))$ for all $i \in S$. We refer to the $1$-core simply as the core.  
\end{definition}

Given a clustering $C$, if there exists a group $S$ that demonstrates a violation of the $\alpha$-core guarantee, i.e., $S$ has size at least $\sfrac{n}{k}$ and the loss of each $i \in S$ for $S$ is lower than $\sfrac{1}{\alpha}$ of her own loss under $C$, we say that $S$ \emph{deviates} under $C$ and refer to it as the \emph{deviating coalition}. We begin by proving a simple result that no finite approximation of the core can be guaranteed for arbitrary losses.
\begin{restatable}{theorem}{corearblb}\label{thm:core-arb-lb}
    For  arbitrary losses, there exists an instance in which no $\alpha$-core clustering exists for any finite $\alpha$.
\end{restatable}
\begin{proof}
Consider an instance with a set of $n=4$ agents $\set{0,1,2,3}$ and $k=2$. Note that any group of at least $2$ agents deserves to form a cluster. For $i\in \set{0,1,2}$, the loss function of agent $i$ is given by
\[
\ell_i(S) = 
\begin{cases}
    \infty & \mbox{if } S = \set{0,1,2} \mbox{ or } 3 \in S,\\
    1 & \mbox{if } |S|=2 \mbox{ and } i+1 \bmod 3 \notin S,\\
    0 & \mbox{if } S = \{i,i+1 \bmod 3\}.
\end{cases}
\]
In words, agent $i \in \set{0,1,2}$ only wants to be in a cluster of size $2$ that does \emph{not} include the undesirable agent $3$; any other cluster has infinite loss. In an ideal such cluster (loss $0$), agent $0$ prefers to be with agent $1$, agent $1$ prefers to be with agent $2$, and agent $2$ prefers to be with agent $0$. The remaining clusters of size $2$ have loss $1$. 

Consider any clustering $C=(C_1,C_2)$. Without loss of generality, say $3 \in C_1$. We take three cases.
\begin{enumerate}
    \item If $|C_1| = 1$, then $C_2 = \set{0,1,2}$. Then, a group $S$ containing any two agents from $C_2$ can deviate, and each $i \in S$ would improve from infinite loss to finite loss.
    \item If $|C_1| \ge 3$, then a group $S$ containing two agents from $C_1$ other than agent $3$ can deviate, and each $i \in S$ would improve from infinite loss to finite loss. 
    \item Suppose $|C_1| = 2$ and let $C_1 \cap \set{0,1,2} = \set{i}$. Then, the group $S = \set{i,(i-1) \bmod 4}$ can deviate: agent $i$ would improve from infinite loss to finite loss, and agent $(i-1) \bmod 4$ would improve from a loss of $1$ to a loss of $0$.
\end{enumerate} 
In each case, every deviating agent improves by an infinite factor, yielding the desired result.  
\end{proof}

Next, for the more structured average loss function, we prove that the core can still be empty, albeit there is now room for a finite approximation. Its proof involves a more intricate construction.

\begin{restatable}{theorem}{coreavglb}\label{thm:core-avg-lb}
    For the average loss, there exists an instance in which no $\alpha$-core clustering exists for $\alpha < \frac{1+\sqrt{3}}{2} \approx 1.366$.
\end{restatable}
\begin{proof}
Let us construct an instance with an even number $k \geq 2$ of clusters. Let $\e=\frac{1+\sqrt{3}}{2}-\alpha$. We set the number of agents to be a multiple of $k$ such that $n\geq k\cdot \max\left\{\frac{1}{2\e}+\frac{1}{2},4\alpha^2\right\}$. Our construction has $k/2+1$ {\em areas}, each consisting of a few locations (points), with several agents placed on each of them. In particular, area $0$ has a single location $M_0$ with $k/2$ agents. For $i=1, 2, ..., k/2$, area $i$ consists of location $M_i$ hosting a single agent, a left location $L_i$ and a right location $R_i$ each hosting $\sfrac{n}{k}-1$ agents. We use $L_i$, $R_i$, and $M_i$ to denote both the corresponding points as well as the set of agents located in them. For $i=1, 2, ..., k/2$, the distance between points $L_i$ and $R_i$ is $1$ while both points are at distance $\frac{n}{2k\alpha}$ from point $M_i$. The distance between any two points in different areas is infinite. 

Consider a $k$-clustering $C$ of the agents. We call {\em bad} any cluster of $C$ that contains agents from different areas; notice that all points in such a cluster have infinite cost. A {\em good} cluster has all its points in the same area and, hence, all the agents contained in it have finite cost. Notice that $C$ has at most $k-1$ good clusters that contain points from areas $1$, $2$, ..., $k/2$. Among these areas, let $t$ be the one with the minimum number of good clusters. Thus, area $t$ either has all its agents in bad clusters or contains one good cluster that includes some of its agents. If at least $\sfrac{n}{k}$ of its agents belong to bad clusters in $C$, a deviating coalition of them would improve their cost from infinite to finite. So, in the following, we assume that clustering $C$ contains exactly one good cluster with at least $\sfrac{n}{k}$ agents from area $t$. 

We distinguish between three cases. The first one is when the good cluster does not contain the agent in $M_t$. Among $R_t$ and $L_t$, assume that $L_t$ has at most as many agents in the good cluster as $R_t$ (the other subcase is symmetric). Then, the cost of all agents of $L_t$ in the good cluster is at least $1/2$. The deviating coalition consisting of all agents in $L_t$ and the agent of $M_t$ (i.e., $\sfrac{n}{k}$ agents in total) improves the cost of all agents by a multiplicative factor at least $\alpha$. Indeed, the cost of the agent in $M_t$ improves from infinite to finite while the cost of any agent in $L_t$ improves from at least $1/2$ to $\frac{1}{2\alpha}$, since any such agent has distance $\frac{n}{2k\alpha}$ to the agent in $M_t$, and is colocated with the other agents in the deviating coalition.

The second case is when the good cluster contains all agents in area $t$. In this case, the cost of the agents in $L_t$ and $R_t$ is $\frac{\frac{n}{2k\alpha}+\sfrac{n}{k}-1}{2\sfrac{n}{k}-1}\geq \frac{1}{4
\alpha}+\frac{1}{2}-\frac{1}{2(2\sfrac{n}{k}-1)}\geq \frac{1}{4\alpha}+\frac{1}{2}-\frac{\e}{2}$ (the second inequality follows by the definition of $n$). Then, each agent in the deviating coalition containing $\frac{n}{2k}$ agents from $L_t$ and $\frac{n}{2k}$ agents from $R_t$ improves their cost to $1/2$, i.e., by a factor of at least $\frac{1}{2\alpha}+1-\e$.

The third case is when the good cluster contains the agent in $M_t$ but does not contain some agent $i$ from $L_t$ or $R_t$. We will assume that agent $i$ belongs to $L_t$ (the other subcase is symmetric). Notice that the cost of the agents in $R_t$ is at least $\frac{\frac{n}{2k\alpha}}{2\sfrac{n}{k}-2}\geq \frac{1}{4\alpha}$. To see why, notice that the claim is trivial for those agents of $R_t$ that belong to bad clusters while each of the agents of $R_t$ in the good cluster is at distance $\frac{n}{2k\alpha}$ to the agent in $M_t$ and there are at most $2\sfrac{n}{k}-2$ in the cluster. The deviating coalition of all agents in $R_t$ together with $i$ decreases their cost to just $k/n$, i.e., by a factor of at least $\frac{n}{4k\alpha}\geq \alpha$ (the inequality follows by the definition of $n$), while the cost of agent $i$ improves from infinite to finite. 

So, there is always a deviating coalition of at least $\sfrac{n}{k}$ agents with each of them improving their cost by a multiplicative factor of $\min\left\{\alpha,1+\frac{1}{2\alpha}-\e\right\}=\alpha$, as desired. The last equality follows by the definition of $\alpha$ and $\e$.
\end{proof}

To complement \Cref{thm:core-arb-lb,thm:core-avg-lb}, we show the existence of a clustering in the $O(\sfrac{n}{k})$-core (resp., $2$-core) for the average (resp., maximum) loss. Despite significant effort, we are unable to determine whether the core is always non-empty for the maximum loss, or whether a constant approximation of the core can be guaranteed for the average loss, which we leave as tantalizing open questions.
\begin{openq}
    For the maximum loss, does there always exist a clustering in the core? 
\end{openq}
\begin{openq}
For the average loss, does there always exist a clustering in the $\alpha$-core for some constant $\alpha$?
\end{openq}

\begin{algorithm}[t]
\caption{\gcc{}$(\alg)$}\label{alg:gcc}
    \KwIn{Set of agents $N$, metric $d$, number of clusters $k$}
    \KwOut{$k$-clustering $C=(C_1,\ldots C_k)$}
    $N'\gets N$\tcp*{Remaining set of agents}
    $j\gets 1$\tcp*{Current cluster number}
    \While{$N'\neq \emptyset$}{
        $C_j \gets \alg(N', d, \ceil{\sfrac{n}{k}})$\tcp*{Find and remove the next cohesive cluster}
        $N'\gets N' \setminus C_j$\;
        $j \gets j+1$\;
    }
    $C_j,C_{j+1},\ldots,C_k \gets \emptyset$\;
    \Return $C=(C_1,\ldots,C_k)$\;
\end{algorithm}

\begin{algorithm}[t]
\caption{\gcsub}\label{alg:smallest-agent-ball}
    \KwIn{Subset of agents $N' \subseteq N$, metric $d$, threshold integer $\tau$}
    \KwOut{Cluster $S$}
    \lIf{$|N'|\leq \tau$}{\Return $S\gets N'$}
    \For{$i \in N'$}{
        $\ell_i\gets \tau$-th closest agent in $N'$ to agent $i$\tcp*{Ties are broken arbitrarily}
        $r_i\gets d(i,\ell_i)$\tcp*{Smallest ball centered at agent $i$ capturing at least $\tau$ agents}
    }
    $i^*\gets\argmin_{i\in N'} r_i$\;
    \Return $S\gets$ the set of $\tau$ closest agents in $N'$ to agent $i^*$\;
\end{algorithm}

\paragraph{Our algorithms.} For the positive result, we design a simple greedy algorithm, \gcc (\Cref{alg:gcc}). It uses a subroutine $\alg$, which, given a subset of agents $N' \subseteq N$, metric $d$, and threshold $\tau$, finds a ``cohesive'' cluster $S$. Here, the term ``cohesive'' is informally used, but we will see a formalization in the next section. The threshold $\tau$ is meant to indicate the smallest size at which a group of agents deserve to form a cluster, but $\alg$ can return a cluster of size greater, equal, or less than $\tau$. 

The algorithm we use as $\alg$ in this section is given as \gcsub (\Cref{alg:smallest-agent-ball}). It finds the smallest ball centered at agent that captures at least $\tau$ agents, and returns a set of $\tau$ agents from this ball. We call this algorithm with the natural choice of $\tau = \ceil{\sfrac{n}{k}}$, so \gcc{}(\gcsub) iteratively finds the smallest agent-centered ball containing $\ceil{\sfrac{n}{k}}$ agents and removes $\ceil{\sfrac{n}{k}}$ in that ball, until fewer than $\ceil{\sfrac{n}{k}}$ agents remain, at which point all remaining agents are put into one cluster and any remaining clusters are left empty.

Overall, \gcc{}(\gcsub) is an adaptation of the \gc algorithm proposed by \citet{chen2019proportionally} for centroid clustering with two key differences in our non-centroid case: (i) while they grow balls centered at feasible cluster center locations, we grow balls centered at the agents, and (ii) while they continue to grow a ball that already captured $\ceil{\sfrac{n}{k}}$ agents (and any agents captured by this ball in the future are placed in the same cluster), we stop a ball as soon as it captures $\ceil{\sfrac{n}{k}}$ agents, which is necessary in our non-centroid case.\footnote{In centroid clustering, additional agents captured later on do not change the loss of the initial $\ceil{\sfrac{n}{k}}$ agents captured as loss is defined by the distance to the cluster center, which does not change. However, in non-centroid clustering, additional agents can change the loss of the initially captured agents, even from zero to positive, causing infinite core approximation when these agents deviate.} Nonetheless, due to its significant resemblance, we refer to the particular instantiation \gcc{}(\gcsub) as \gc hereinafter. The following result is one of our main results.

\begin{restatable}{theorem}{coreub}\label{thm:core-ub}
    For the average (resp., maximum) loss, the \gc algorithm is guaranteed to return a clustering in the $(2\cdot \ceil{\sfrac{n}{k}}-3)$-core (resp., $2$-core) in $O(kn)$ time complexity,  and these bounds are (almost) tight. 
\end{restatable}
\begin{proof}
    Let $C=\{C_1,\ldots C_k\}$ be the $k$-clustering returned by \gc. Let $S\subseteq N$ be any set of at least $\sfrac{n}{k}$ agents such that their average loss satisfies
    \begin{align}
        \ell_{i} (C(i))> (2\cdot \ceil{\sfrac{n}{k}}-3) \cdot \ell_i(S),\label{ineq:core-aver-1}
    \end{align}
for every $i\in S$. 

Let $i^*$ be the agent that was the first among the agents in $S$ that was included in some cluster by the algorithm. Consider the time step before this happens and let $i'\in C(i^*)$ be the agent that had the minimum distance $R$ from the $\lceil \sfrac{n}{k}\rceil$-th agent in $C(i^*)$ among all agents that had not been included to clusters by the algorithm before. Then, 
\begin{align}\nonumber
    \ell_{i^*}(C(i^*)) &= \frac{1}{|C(i^*)|} \sum_{i\in C(i^*)}{d(i^*,i)}\\\nonumber
    &\leq \frac{1}{|C(i^*)|}\left(d(i^*,i')+\sum_{i\in C(i^*)\setminus \{i',i^*\}}{(d(i^*,i')+d(i',i))}\right)\\\label{ineq:core-aver-2}
    &\leq \left(2-\frac{3}{\lceil \sfrac{n}{k}\rceil}\right)\cdot R.
\end{align}
The first inequality follows by applying the triangle inequality. The second inequality follows since $C(i^*)$ has $\lceil \sfrac{n}{k}\rceil$ agents and, thus, the RHS has $2\lceil \sfrac{n}{k}\rceil-3$ terms representing distances of agents in $C(i^*)$ from agent $i'$, each bounded by $R$.

Now, recall that, at the time step the algorithm includes cluster $C(i^*)$ in the clustering, none among the (at least $\lceil \sfrac{n}{k}\rceil$) agents of $S$ have been included in any clusters. Then, $S$ contains at most $\lceil \sfrac{n}{k}\rceil -1$ agents located at distance less than $R$ from agent $i^*$; if this were not the case, the algorithm would have included  agent $i^*$ together with $\lceil \sfrac{n}{k}\rceil-1$ other agents of $S$ in a cluster instead of the agents in $C(i^*)$. Thus, $S$ contains at least $|S|-\lceil \sfrac{n}{k}\rceil +1$ agents at distance at least $R$ from agent $i^*$. Thus,
\begin{align}\label{ineq:core-aver-3}
    \ell_{i^*}(S) &=\frac{1}{|S|}\sum_{i\in S}{d(i^*,i)}\geq \frac{|S|-\lceil \sfrac{n}{k}\rceil+1}{|S|}\cdot R\geq \frac{1}{\lceil \sfrac{n}{k}\rceil} \cdot R.
\end{align}
The second inequality follows since $|S|\geq \lceil \sfrac{n}{k}\rceil$. Now, \Cref{ineq:core-aver-2} and \Cref{ineq:core-aver-3} yield $\ell_{i^*} (C(i^*))\leq (2\cdot \ceil{\sfrac{n}{k}}-3) \cdot \ell_{i^*}(S)$, contradicting \Cref{ineq:core-aver-1}. 

Now, assume that there exists a set $S\subseteq N$ of at least $\sfrac{n}{k}$ agents such that their maximum loss satisfies
    \begin{align}\label{ineq:core-max-1}
        \ell_{i} (C(i))> 2 \cdot \ell_i(S),
    \end{align}
for every $i\in S$. Again, let $i^*$ be the agent that was the first among the agents in $S$ that was included in some cluster by the algorithm. Consider the time step before this happens and let $i'\in C(i^*)$ be the agent that had the minimum distance $R$ from the $\lceil \sfrac{n}{k}\rceil$-th agent in $C(i^*)$ among all agents that had not been included to clusters by the algorithm before. Then, the maximum loss of agent $i^*$ for cluster $C(i^*)$ is
\begin{align}\label{ineq:core-max-2}
    \ell_{i^*}(C(i^*)) &= \max_{i\in C(i^*)}{d(i^*,i)} \leq \max_{i\in C(i^*)}{(d(i^*,i')+d(i',i))}\leq 2\cdot R.
\end{align}
The first inequality follows by applying the triangle inequality and the second one since all agents in $C(i^*)$ are at distance at most $R$ from agent $i'$. We also have
\begin{align}\label{ineq:core-max-3}
    \ell_{i^*}(S)&=\max_{i\in S}{d(i^*,i)}\geq R,
\end{align}
otherwise, the algorithm would include a subset of $\lceil \sfrac{n}{k}\rceil$ agents from set $S$ in the clustering instead of $C(i^*)$. Together, \Cref{ineq:core-max-2} and \Cref{ineq:core-max-3} contradict \Cref{ineq:core-max-1}. This completes the proof of the upper bounds.

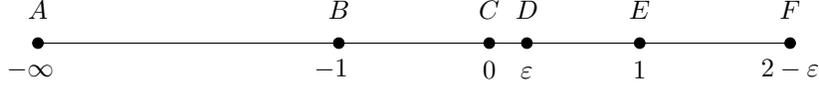
\begin{figure}[t]
    \centering
    \begin{tikzpicture}
        % Axis
        \draw[-] (0,0) -- (10,0);
        \node[above] at (-0.1,-0.6) {$-\infty$};
        \node[above] at (0,0.2) {$A$};
        \filldraw[black] (0,0) circle (2pt);
        \node[above] at (3.9,-0.6) {$-1$};
        \node[above] at (4,0.2) {$B$};
        \filldraw[black] (4,0) circle (2pt);

        \node[above] at (6,-0.6) {$0$};
        \node[above] at (6,0.2) {$C$};
        \filldraw[black] (6,0) circle (2pt);

           \node[above] at (6.5,-0.6) {$\e$};
        \node[above] at (6.5,0.2) {$D$};
        \filldraw[black] (6.5,0) circle (2pt);

           \node[above] at (8,-0.6) {$1$};
        \node[above] at (8,0.2) {$E$};
        \filldraw[black] (8,0) circle (2pt);

           \node[above] at (10,-0.6) {$2-\e$};
        \node[above] at (10,0.2) {$F$};
        \filldraw[black] (10,0) circle (2pt);
        
    \end{tikzpicture}
    \caption{The instance used to show the lower bounds in~\Cref{thm:core-ub} and~\Cref{lem:gcsub-gccsub}. }
    \label{fig:lower-bound}
\end{figure}

We now show that the analysis is tight  for both the average and the maximum loss functions using the instance depicted in~\Cref{fig:lower-bound} with one agent at locations $A$, $D$, and $E$, two agents at location $B$, $n/2-3$ agents at location $C$, and $n/2-2$ agents at location $F$.  
Suppose that $k=2$. It is easy to see that \gc returns a $2$-clustering with the agents located at points $A$, $B$, and $C$ in one cluster and the agents located at points $D$, $E$, and $F$ in another. Notice that the agents at locations $B$ and $C$ have infinite loss under both loss functions, while the agent located at position $D$ has maximum loss $2(1-\e)$ and average loss $\frac{(n-3)(1-\e)}{n/2}$. Now, consider the deviating coalition of the $n/2$ agents at locations $B$, $C$, and $D$. The agents at $B$ and $C$ improve their loss from infinite to finite, while the agent located at $C$ improves her maximum loss to $1+\e$ and her average loss to $\frac{2+(n/2-1)\e}{n/2}$, for multiplicative improvements approaching $2$ and $n/2-3/2$ as $\e$ approaching $0$.

Since \gc calls \gcsub at most $k$ times and \gcsub does at most $n$ iterations in each call, we easily see that the time complexity of  \gc  is $O(kn)$. 
\end{proof}

In many applications of clustering, such as clustered federated learning, the average loss is realistic because the agent's loss depends on all the agents in her cluster, and not just on a single most distant agent. Hence, it is a little disappointing that the only approximation to the core that we are able to establish in this case is $\alpha = O(\sfrac{n}{k})$, which is rather unsatisfactory. We demonstrate two ways to circumvent this negative result. First, we consider demanding that any deviating coalitions be of size at least $\delta \cdot \sfrac{n}{k}$ for some $\delta > 1$. In \Cref{app:core-bicriteria}, we show that any \emph{constant} $\delta > 1$ reduces the approximation factor $\alpha$ to a constant. In the next section, we explore a different approach: we relax the core to a slightly weaker fairness guarantee, which we show can be satisfied exactly, even under arbitrary losses. 

\section{Fully Justified Representation}\label{sec:fjr}

\citet{peters2021proportional} introduced \emph{fully justified representation} (FJR) as a relaxation of the core in the context of approval-based committee selection. The following definition is its adaptation to non-centroid clustering. Informally, for a deviating coalition $S$, the core demands that the loss $\ell_i(S)$ of each member $i$ after deviation be lower than \emph{her own loss before deviation}, i.e., $\ell_i(C(i))$. FJR demands that it be lower than the \emph{minimum loss of any member before deviation}, i.e., $\min_{j \in S} \ell_j(C(j))$.

\begin{definition}[$\alpha$-Fully Justified Representation ($\alpha$-FJR)]
    For $\alpha \ge 1$, a $k$-clustering $C=(C_1,\ldots,C_k)$ satisfies $\alpha$-fully  justified representation ($\alpha$-FJR) if there is no group of agents $S \subseteq N$ with $|S| \ge \sfrac{n}{k}$ such that $\alpha \cdot \ell_i(S) < \min_{j \in S} \ell_j(C(j))$ for each $i \in S$, i.e., if $\alpha \cdot \max_{i \in S} \ell_i(S) < \min_{j \in S} \ell_j(C(j))$. We refer to $1$-FJR simply as FJR.  
\end{definition}

We easily see that $\alpha$-FJR is a relaxation of $\alpha$-core.
\begin{proposition}
    For $\alpha \ge 1$, $\alpha$-core implies $\alpha$-FJR for arbitrary loss functions.
\end{proposition}
\begin{proof}
    Suppose that a clustering $C$ is in the $\alpha$-core. Thus, for every $S \subseteq N$ with $|S| \geq \sfrac{n}{k}$, there exists $i \in S$ for which $\alpha \cdot \ell_i(S) \ge \ell_i(C(i)) \ge \min_{j \in S} \ell_j(C(j))$, so the clustering is also $\alpha$-FJR.
\end{proof}

\subsection{Arbitrary Loss Functions}

We prove that an (exactly) FJR clustering is guaranteed to exist, even for arbitrary losses. For this, we need to define the following computational problem. 

\begin{definition}[\gccsub]
  Given a set of agents $N$ and a threshold $\tau$, the \gccsub problem asks to find a cluster $S \subseteq N$ of size at least $\tau$ such that the maximum loss of any $i \in S$ for $S$ is minimized, i.e., find $\argmin_{S \subseteq N' : |S| \ge \tau} \max_{i \in S} \ell_i(S)$. 
  
  For $\lambda \ge 1$, a $\lambda$-approximate solution $S$ satisfies $ \max_{i \in S} \ell_i(S) \le \lambda \cdot \max_{i \in S'} \lambda_i(S')$ for all $S' \subseteq N$ with $|S'| \ge \tau$, and a $\lambda$-approximation algorithm returns a $\lambda$-approximate solution on every instance.
\end{definition}

We show that plugging in a $\lambda$-approximation algorithm $\alg$ to the \gccsub problem into the \gcc algorithm designed in the previous section yields a $\lambda$-FJR clustering. In order to work with arbitrary losses, we need to consider a slightly generalized \gcc algorithm, which takes the loss functions $\ell_i$ as input instead of a metric $d$, and passes these loss functions to algorithm $\alg$. 

\begin{restatable}{theorem}{fjrexact}\label{thm:fjr-exact}
    For arbitrary losses, $\alpha \ge 1$, and an $\alpha$-approximation algorithm $\alg$ for the \gccsub problem, \gcc{}$(\alg)$ is guaranteed to return a $\alpha$-FJR clustering. Hence, an (exactly) FJR clustering is guaranteed to exist.
\end{restatable}
\begin{proof}
    Suppose for contradiction that the $k$-clustering $C=\set{C_1,\ldots C_k}$ returned by \gcc{}$(\alg)$ on an instance is not $\alpha$-FJR. Then, there exists a group $S \subseteq N$ with $|S| \geq \sfrac{n}{k}$ such that $\alpha \cdot \max_{i \in S} \ell_i(S) < \min_{i \in S} \ell_i(C(i))$. Let $i^*$ be  the first agent in $S$ that was assigned to a cluster during the execution of \gcc, by calling $\alg$ on a subset of agents $N'$. Note that $S\subseteq N'$. Then, we have that $
         \max_{i\in C(i^*)} \ell_i(C(i^*)) \geq \ell_{i^*}(C(i^*)) > \alpha \cdot \max_{i \in S} \ell_i(S),
     $
     which contradicts $\alg$ being an $\alpha$-approximation algorithm for the \gccsub problem. Hence, \gcc{}$(\alg)$ must return an $\alpha$-FJR clustering.

     Using an exact algorithm $\alg$ for the \gccsub problem (e.g., the inefficient brute-force algorithm), we get that a $1$-FJR clustering is guaranteed to exist. 
\end{proof}

\subsection{Average and Maximum Loss Functions}
Let $\alg^*$ be an exact algorithm for the \gccsub problem for the average (resp., maximum) loss. First, we notice that we cannot expect it to run in polynomial time, even for these structured loss functions. This is because it can be used to detect whether a given undirected graph admits a clique of at least a given size,\footnote{\label{fnt:reduction}To detect whether an undirected graph $G = (V,E)$ has a clique of size at least $t$, we run $\alg^*$ with each node being an agent, the distance between two agents being $1$ if they are neighbors and $2$ otherwise, and $k = n/t$. A clique of size at least $t$ exists in $G$ if and only if a cluster $S$ exists with $|S| \ge n/k = t$ and $\max_{i \in S} \ell_i(S) = 1$.} which is an NP-complete problem. Hence, \gcc{}$(\alg^*)$ is an inefficient algorithm.

One can easily check that the proof of \Cref{thm:core-ub} extends to show that it achieves not only $1$-FJR (\Cref{thm:fjr-exact}), but also in the $O(n/k)$-core (resp., $2$-core) for the average (resp., maximum) loss. For the core, \gc is an obvious improvement as it achieves the same approximation ratio but in polynomial time. For FJR, we show that \gc still achieves a constant approximation in polynomial time. We prove this by showing that the \gcsub algorithm used by \gc achieves the desired approximation to the \gccsub problem, and utilizing \Cref{thm:fjr-exact}. 

\begin{restatable}{lemma}{gcsubgccsub}\label{lem:gcsub-gccsub}
    For the average (resp., maximum) loss, \gcsub is a $4$-approximation (resp., $2$-approximation) algorithm for the \gccsub problem, and this is tight.
\end{restatable}
\begin{proof}
    For some $N'\subseteq N$, let $S$ be the most cohesive cluster  and  let $S'\neq S$ be the cluster that \gcsub returns. Suppose that $S'$ consists by the $\ceil{\sfrac{n}{k}}$ closest agents in $N'$ to some agent $i^*$ and the distance of $i^*$ to her $\ceil{\sfrac{n}{k}}$-th  closest agent in $N'$ is equal to $R$. From the triangle inequality, we get that every two agents in $S'$ have distance at most $2R$, and therefore, under both loss functions, we get that  $max_{i\in S'}\ell_i(S')\leq 2\cdot R$.

     We show  that there are  two individuals in $S$,  $i_1$ and $i_2$, such that the $d(i_1,i_2)\geq R$. Indeed if for each $i,i'\in S$, $d(i,i')<R$, then $i^*$ would  not be the agent in $N'$ with the smallest distance to her $\ceil{\sfrac{n}{k}}$-th  closest agent in $N'$ and  \gcsub would not return $S$.  From this fact, we immediately get a $2$-approximation for the maximum loss, since  $max_{i\in S}\ell_i(S)\geq  R$.

     Now, for the average cost, note that 
    \begin{align*}
       |S|\cdot  d(i_1,i_2) = \sum_{i\in S} d(i_1,i_2)
        \leq
        \sum_{i\in S} \left( d(i_1,i)+ d(i,i_2) \right).
    \end{align*}
    From this we get, that either $\sum_{i\in S}  d(i_1,i)\geq |S|\cdot  d(i_1,i_2) /2 $  or $\sum_{i\in S}  d(i_2,i)\geq |S|\cdot  d(i_1,i_2) /2 $. Therefore, either $\ell_{i_1}(S)\geq d(i_1,i_2)/2\geq R/2$  or  $\ell_{i_2}(S)\geq d(i_1,i_2)/2\geq R/2$.  This means that $max_{i\in S}\ell_i(S)\geq  R/2$ and the lemma follows.

    Next, we show that there are  instances for which \gcsub achieves exactly  these bounds. Consider the instance showing in~\Cref{fig:lower-bound},
 For $k=2$, suppose there are = $1$ point at position $A$, $n/4$ points at position $B$,   $n/4-1$ at position $C$,  $1$ point at position $D$,   $1$ point at position $E$ at position $D$, and $n/2-2$ points at position $F$.  It is not hard to see that \gcsub will return the cluster $S=\{D,E,F\}$. But  $S'=\{B,C,D\}$ can reduce the average loss by a factor equal to $4$ and the maximum loss by a factor equal to $2$ as $n$ grows and $\epsilon$ goes to $0$.  
\end{proof}

Plugging in \Cref{lem:gcsub-gccsub} into \Cref{thm:fjr-exact}, we get the following.

\begin{restatable}{corollary}{fjrgc}\label{cor:fjr-gc}
    The (efficient) \gc algorithm is guaranteed to return a clustering that is $4$-FJR (resp., $2$-FJR) for the average (resp., maximum) loss.  
\end{restatable}

Determining the best FJR approximation achievable in polynomial time remains an open question. 

\begin{openq}
    For the average (or maximum) loss, what is the smallest $\alpha$ for which an $\alpha$-FJR clustering can be computed in polynomial time, assuming P $\neq$ NP?
\end{openq}

Also, while \Cref{thm:fjr-exact} shows that exact FJR is achievable for the average and maximum losses, a single clustering may not achieve FJR for both losses simultaneously (the algorithm used in \Cref{thm:fjr-exact} depends on the loss function). In contrast, \gc does not depend on whether we are using the average or the maximum loss. Thus, the clustering it produces is simultaneously $4$-FJR for the average loss and $2$-FJR for the maximum loss (\Cref{cor:fjr-gc}); this is novel even as an existential result, ignoring the fact that it can be achieved using an efficient algorithm \gc. We do not know how much this existential result can be improved upon. 

\begin{openq}
    What is the smallest $\alpha$ such that there always exists a clustering that is simultaneously $\alpha$-FJR for both the average loss and the maximum loss?
\end{openq}
    
\subsection{Auditing FJR}\label{sec:audit}

Next, we turn to the question of auditing the FJR approximation of a given clustering. In particular, the goal is to find the maximum FJR  violation of a clustering  $C$, i.e., the largest $\alpha$ for which there exists a group of agents of size at least $\sfrac{n}{k}$ such that, if they were to form their own cluster, the loss of each of them would be lower, by a factor of at least $\alpha$,  than the minimum loss of any of them under clustering $C$. Because guarantees such as the core and FJR are defined with exponentially many constraints, it is difficult to determine the exact approximation ratio achieved by a given solution efficiently, which is why prior work has not studied auditing for proportional fairness guarantees. Nonetheless, we show that the same ideas that we used to \emph{find} an (approximately) FJR clustering can also be used to (approximately) audit the FJR approximation of any given clustering. 

\begin{definition}[$\lambda$-Approximate FJR Auditing]
We say that algorithm $\alg$ is a $\lambda$-approximate FJR auditing algorithm if, given any clustering $C$, it returns $\theta$ such that the exact FJR approximation of $C$ (i.e., the smallest $\alpha$ such that $C$ is $\alpha$-FJR) is in $[\theta,\lambda \cdot \theta]$.
\end{definition}

\begin{algorithm}[h]
\caption{\auditfjr{}$(\alg)$}\label{alg:audit-FJR}
%\SetAlgoNoLine
\KwIn{Set of agents $N$, metric $d$, number of clusters $k$, clustering $C$}
\KwOut{Estimate $\theta$ of the FJR approximation of $C$}
$N' \gets N;  \theta \gets 0$\tcp*{Remaining agents, current FJR apx estimate}
\While{$|N'|\geq  \sfrac{n}{k}$}{
    $S \gets \alg(N',d,\sfrac{n}{k})$\tcp*{Find a cohesive group $S$}
    $\theta \gets \max\set{\theta,  \frac{\min_{i\in S} \ell_i(C(i))}{\max_{i\in S}\ell_i(S)}}$\tcp*{Update $\theta$ using the FJR violation due to $S$}
    $i^*\gets \argmin_{i\in S} \ell_i(C(i)) $\;
    $N' \gets N' \setminus \set{i^*}$\tcp*{Remove the agent with the smallest current loss}
}
\Return $\theta$\;
\end{algorithm}

We design another parametric algorithm, \auditfjr{}$(\alg)$, presented as \Cref{alg:audit-FJR}, which iteratively calls $\alg$ to find a `cohesive' cluster $S$, similarly to \gcc. But while \gcc removes all the agents in $S$ from further consideration, \auditfjr removes only the agent in $S$ with the smallest loss under the given clustering $C$ from further consideration. Thus, instead of finding up to $k$ cohesive clusters, it finds up to $n$ cohesive clusters. It returns the maximum FJR violation of $C$ demonstrated by any of these $n$ possible deviating coalitions (recall that the exact FJR approximation of $C$ is the maximum FJR violation across all the exponentially many possible deviating coalitions of size at least $\ceil{\sfrac{n}{k}}$). 

We show that if $\alg$ was a $\lambda$-approximation algorithm for the \gccsub problem, then the resulting algorithm is a $\lambda$-approximate FJR auditing algorithm. In particular, if we were to solve the \gccsub problem exactly in each iteration (which would be inefficient), the maximum FJR violation across those $n$ cohesive clusters found would indeed be the maximum FJR violation across all the exponentially many deviating coalitions, an apriori nontrivial insight. Fortunately, we can at least plug in the \gcsub algorithm, which we know achieves constant approximation to the \gccsub problem (\Cref{lem:gcsub-gccsub}).

\begin{restatable}{theorem}{thmauditfjr}\label{thm:audit-fjr}
    For $\lambda \ge 1$, if $\alg$ is a $\lambda$-approximation algorithm to the \gccsub problem, then \auditfjr{}$(\alg)$ is a $\lambda$-approximate FJR auditing algorithm. Given \Cref{lem:gcsub-gccsub}, it follows that for the average (resp., maximum) loss, \auditfjr{}(\gcsub) is an efficient $4$-approximate (resp., $2$-approximate) FJR auditing algorithm. 
\end{restatable}
\begin{proof}
    Suppose $\alg$ is a $\lambda$-approximation algorithm for the \gccsub problem. Consider any clustering $C$ on which \auditfjr{}$(\alg)$ returns $\theta$. Let $\rho = \max_{S \subseteq N : |S| \ge \ceil{\sfrac{n}{k}}} \frac{\min_{i \in S} \ell_i(C(i))}{\max_{i \in S} \ell_i(S)}$ be the exact FJR approximation of $C$. First, it is easy to check that $\rho \ge \theta$ because $\theta$ is computed by taking the maximum of the same expression as $\rho$ is, but over only some (instead of all possible) $S$. Hence, it remains to prove that $\rho \le \lambda \cdot \theta$.

    Consider any group $S \subseteq N$ with $|S|\geq \sfrac{n}{k}$. Let $i^*$ be the first agent in $S$ that was removed by \auditfjr, say when $\alg$ returned a group $S'$ containing it; there must be one such agent because $|S| \ge \sfrac{n}{k}$ and when \auditfjr stops, fewer than $\sfrac{n}{k}$ agents remain in $N'$. Now, we have that
    \[
    \frac{\min_{i\in S} \ell_i(C(i))}{\max_{i\in S}\ell_i(S)}
    \leq \frac{\ell_{i^*}(C(i^*))}{\max_{i\in S}\ell_i(S)}
    \leq \lambda \cdot \frac{\ell_{i^*}(C(i^*))}{\max_{i\in S'} \ell_i(S')}
    = \lambda \cdot \frac{\min_{i \in S'} \ell_i(C(i))}{\max_{i\in S'} \ell_i(S')}
    \leq \lambda \cdot \theta,
    \]
    where the second inequality holds because $\alg$ is a $\lambda$-approximation algorithm for the \gccsub problem, which implies $\max_{i \in S'} \ell_i(S') \le \lambda \cdot \max_{i \in S} \ell_i(S)$; the next equality holds because agent $i^*$ was selected for removal when $S'$ was returned, which implies $i^* \in \argmin_{i \in S'} \ell_i(C(i))$; and the final inequality holds because $\theta$ is  updated to be the maximum of all FJR violations witnessed by the algorithm, and violation due to $S'$ is one of them.   

Finally, using the approximation ratio bound of \gcsub for the \gccsub problem from \Cref{lem:gcsub-gccsub}, we obtain the desired approximate auditing guarantee of \auditfjr{}(\gcsub).
\end{proof}

Unfortunately, the technique from \Cref{thm:audit-fjr} does not extend to auditing the core. This is because it requires upper bounding $\min_{i \in S} \frac{\ell_i(C(i))}{\ell_i(S)}$ (instead of $\frac{\min_{i\in S} \ell_i(C(i))}{\max_{i\in S}\ell_i(S)}$); this can be upper bounded by $\frac{\ell_{i^*}(C(i^*))}{\ell_{i^*}(S)}$, but we cannot lower bound $\ell_{i^*}(S)$. The fact that $\alg$ approximates the \gccsub problem only lets us lower bound $\max_{i \in S} \ell_i(S)$. We leave it as an open question whether an efficient, constant-approximate core auditing algorithm can be devised.
\begin{openq}
    Does there exist a polynomial-time, $\alpha$-approximate core auditing algorithm for some constant $\alpha$?
\end{openq}

For the maximum loss, we can show that our $2$-approximate FJR auditing algorithm is the best one can hope for in polynomial time. The case of average loss remains open. 
\begin{restatable}{theorem}{auditfjrhardness}\label{thm:audit-fjr-hardness}
Assuming P $\neq$ NP, there does not exist a polynomial-time $\lambda$-approximate FJR auditing algorithm for the maximum loss, for any $\lambda < 2$. 
\end{restatable}
\begin{proof}
We show that such an algorithm can be used to solve the CLIQUE problem, which asks whether a given undirected graph $G = (V,E)$ admits a clique of size at least $t$. The problem remains hard with $t \ge 3$, so we can assume this without loss of generality. Given $(G,t)$, we first modify $G = (V,E)$ into $G' = (V',E')$ as follows. To each $v \in V$, we attach $t-2$ new (dummy) nodes, and to one of those dummy nodes, we attach yet another dummy node. In total, for each $v \in V$, we are adding $t-1$ dummy nodes, so the final number of nodes is $|V'| = |V| \cdot t$. 

Next, we create an instance of non-centroid clustering with $n = |V'|$ agents, one for each $v \in V'$. The distance $d(u,v)$ is set as the length of the shortest path between $u$ and $v$. Set $k = |V|$. 

Consider a clustering $C$ in which each real node $v \in V$ is put into a separate cluster, along with the $t-1$ dummy nodes created for it. Note that $\ell_v(C(v)) = 2$ for each real node $v \in V$ (due to the dummy node attached to a dummy node attached to $v$) and $\ell_v(C(v)) \in \set{2,3}$ for each dummy node $v \in V' \setminus V$. Let us now consider possible deviating coalitions $S$. 

If a dummy node $v$ is included in $S$, then in order for an FJR violation, its maximum loss would have to be strictly reduced. If $\ell_v(C(v)) = 2$, then we must have $\ell_v(S) = 1$, but no dummy node has at least $t \ge 3$ nodes within a distance of $1$. If $\ell_v(C(v)) = 3$, then in order to find at least $t$ nodes within a distance of at most $2$ (and the set not be identical to one of the clusters), $S$ must include at least one real node $v'$ that is not associated with the dummy node $v$. However, in this case, $\ell_{v'}(S) \ge 2$ whereas $\ell_{v'}(C(v')) = 2$, so no FJR violation is possible. 

The only remaining case is when $S$ consists entirely of real nodes. Since $\ell_v(C(v)) = 2$ for every real node $v \in V$, an FJR violation exists if an only if $\max_{v \in S} \ell_v(S) = 1$, which happens if and only if $S$ is a clique of real nodes size at least $t$. 

Thus, we have established that the FJR approximation of $C$ is $2$ if there exists a clique of size at least $t$ in $G$, and $1$ otherwise. Since a $\lambda$-approximate auditing algorithm with $\lambda < 2$ can distinguish between these two possibilities, it can be used to solve the CLIQUE problem. \end{proof}

\section{Experiments}\label{sec:expt}

In this section, we empirically compare \gc with the popular clustering algorithms $k$-means++ and $k$-medoids on real data.  Our focus is on the tradeoff between fairness (measured by the core and FJR) and accuracy (measured by traditional clustering objectives) they achieve. 

\paragraph{Datasets.} We consider three different datasets from the UCI Machine Learning Repository~\cite{DG17}, namely Census Income,  Diabetes, and Iris.  For the first two datasets, each data point corresponds to a human being, and it is reasonable to assume that each individual prefers to be clustered along with other similar individuals. We also consider the third dataset for an interesting comparison with the empirical work of~\citet{chen2019proportionally}, who compared the same algorithms but for centroid clustering.

The Census Income dataset contains demographic and economic characteristics of individuals, which are used to predict whether an individual's annual income exceeds a threshold. For our experiments, we keep all the numerical features (i.e. \texttt{age},
\texttt{education-num}, \texttt{capital-gain}, \texttt{capital-loss}, and \texttt{hours-per-week})  along with \texttt{sex}, encoded as binary values. There are in total $32{,}561$ data points, each with a sample weight attribute (\texttt{fnlwgt}). The Diabetes dataset contains numerical features, such as age and blood pressure, for about 768 diabetes patients. The Iris dataset consists of 150 records of numerical features related to the petal dimensions of different types of iris flowers. 

\paragraph{Measures.}  For fairness, we measure the true FJR and core approximations of each algorithm with respect to both the average and maximum losses. For accuracy, we use three traditional clustering objectives: the average within-cluster distance $\sum_{t\in [k]} \frac{1}{|C_t|} \cdot  \sum_{i,j \in C_t} d(i,j)$, termed \emph{cost} by \citet{ahmadi2022individual}, as well as the popular $k$-means and $k$-medoids objectives. 

\paragraph{Experimental setup.} We implement the standard $k$-means++ and $k$-medoids clustering algorithms from the Scikit-learn project\footnote{https://scikit-learn.org}, averaging the values for each measure over $20$ runs, as their outcomes depend on random initializations. The computation of \gc neither uses randomization nor depends on the loss function with respect to which the core or FJR approximation is measured. Since calculating core and FJR approximations requires considerable time, for both the Census Income and Diabetes datasets, we sample $100$ data points and plot the means and standard deviations over $40$ runs. For the former, we conduct weighted sampling according to the \texttt{fnlwgt} feature. We conducted our experiments on a server with 32 cores / 64 threads at 4.2 GHz and 128 GB of RAM.

\begin{figure}[htb!]
    \centering
    \begin{subfigure}[b]{0.45\textwidth}
        \caption{Core violation, average loss}
        \includegraphics[width=\textwidth]{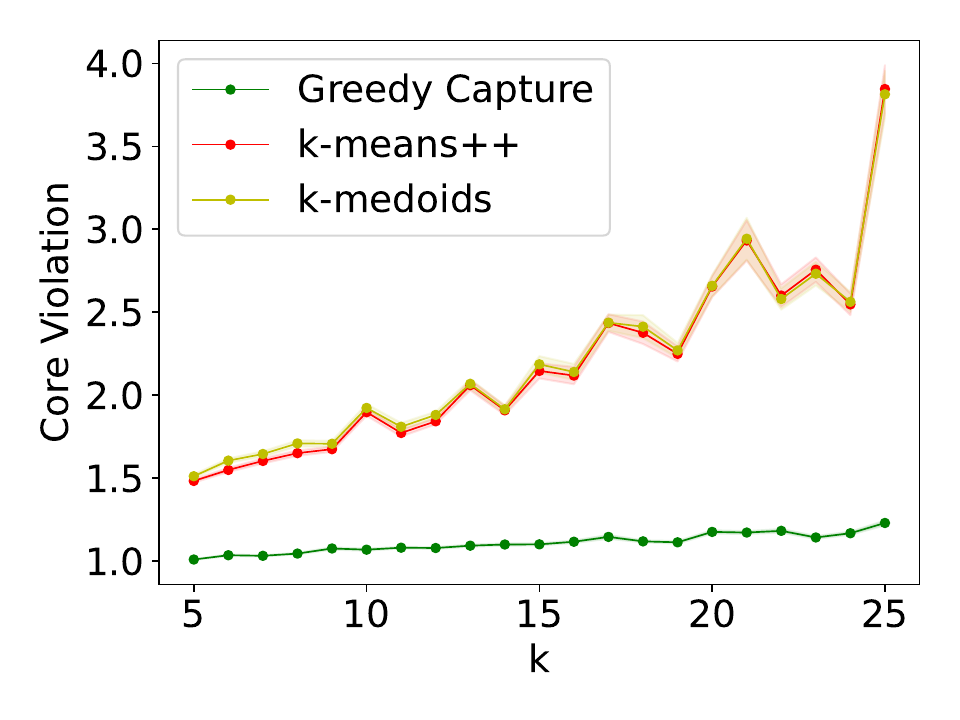}
    \end{subfigure}\hspace{0.02\textwidth}%    
    \begin{subfigure}[b]{0.45\textwidth}
        \caption{FJR violation, average loss}
        \includegraphics[width=\textwidth]{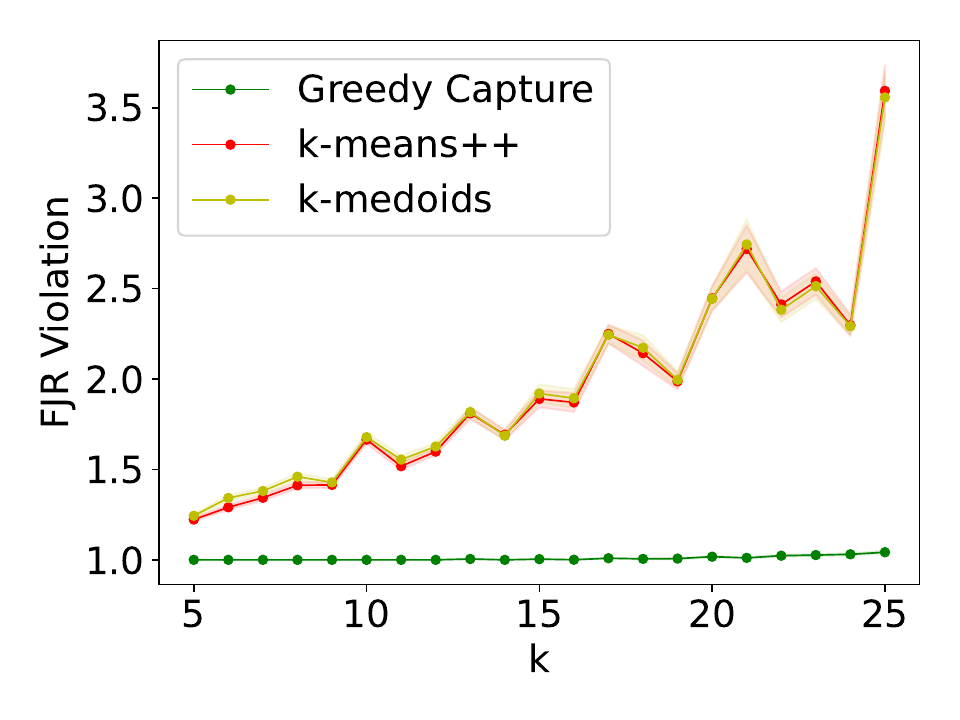}
    \end{subfigure}
    
    \begin{subfigure}[b]{0.45\textwidth}
        \caption{Core violation, maximum loss}
        \includegraphics[width=\textwidth]{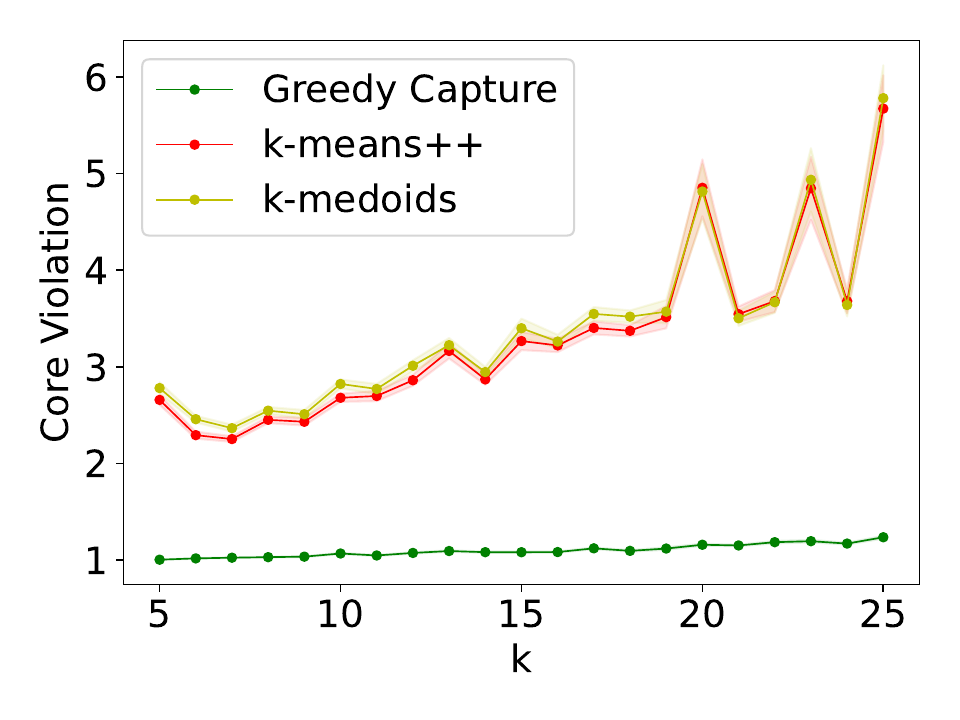}
    \end{subfigure}\hspace{0.02\textwidth}%  
    \begin{subfigure}[b]{0.45\textwidth}
        \caption{FJR violation, maximum loss}
        \includegraphics[width=\textwidth]{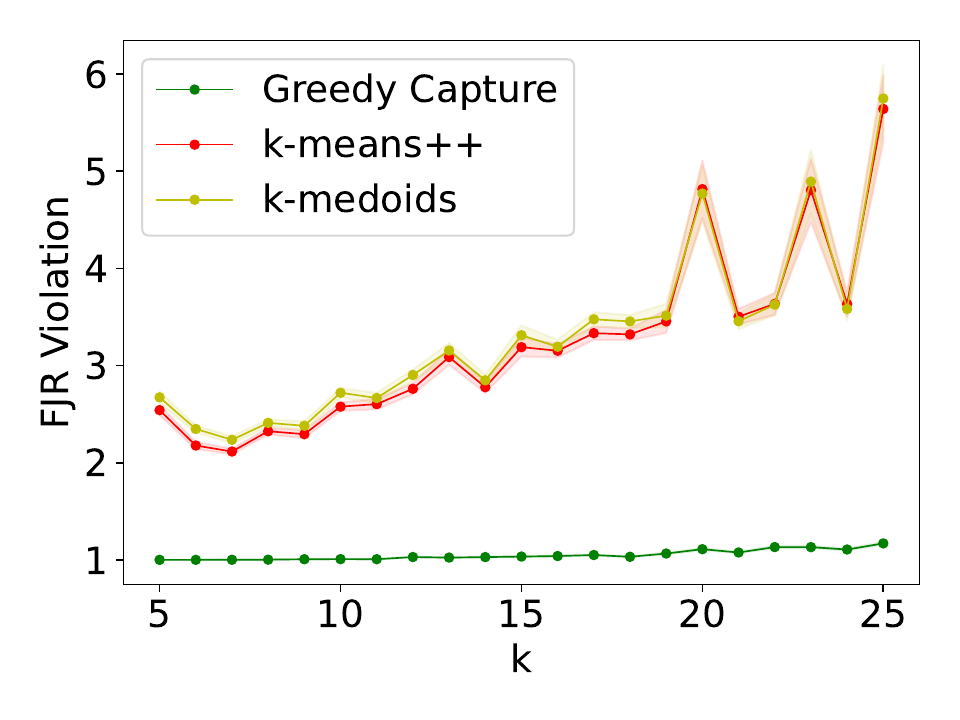}
    \end{subfigure}
    
    \begin{subfigure}[b]{0.33\textwidth}
        \caption{Avg within-cluster distance}
        \includegraphics[width=\textwidth]{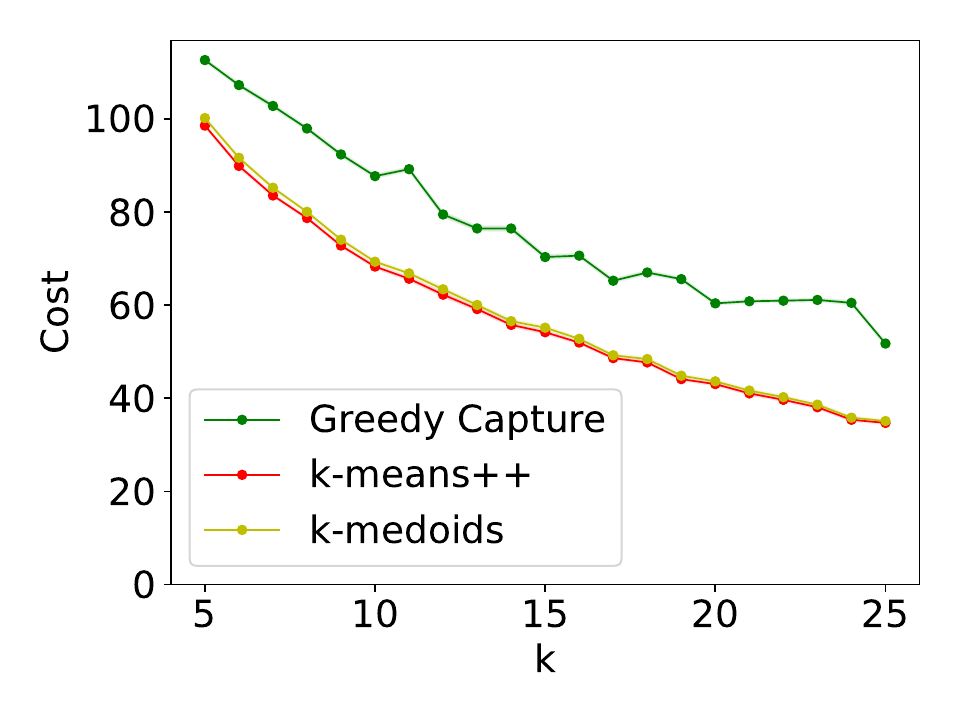}
    \end{subfigure}\hspace{0.003\textwidth}%
    \begin{subfigure}[b]{0.33\textwidth}
        \caption{$k$-means objective}
        \includegraphics[width=\textwidth]{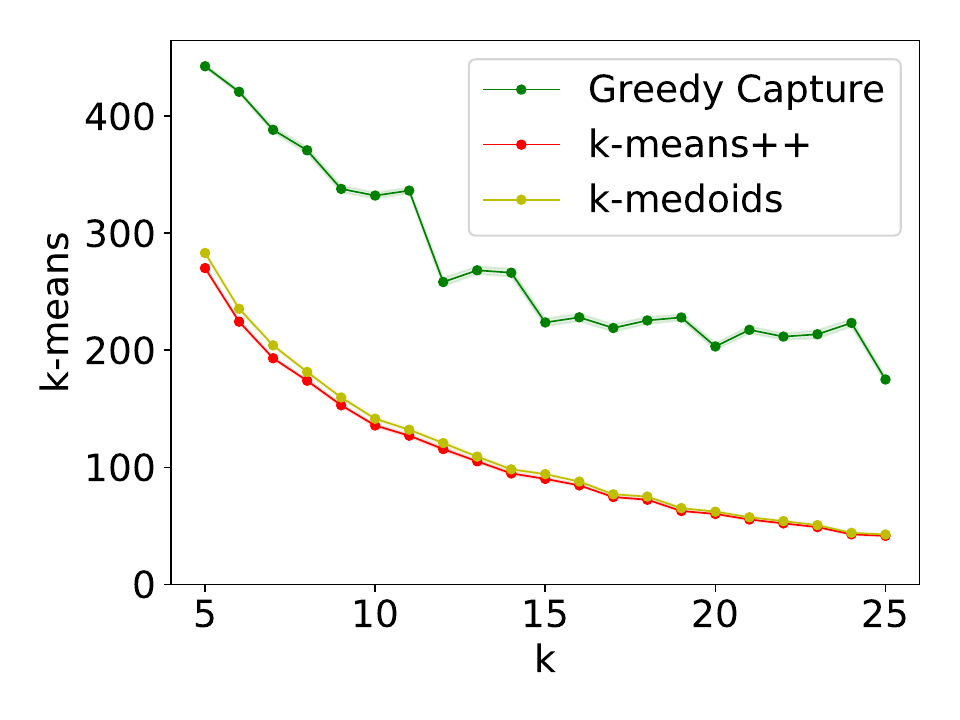}
    \end{subfigure}\hspace{0.003\textwidth}%
    \begin{subfigure}[b]{0.33\textwidth}
    \caption{$k$-medoids objective}
        \includegraphics[width=\textwidth]{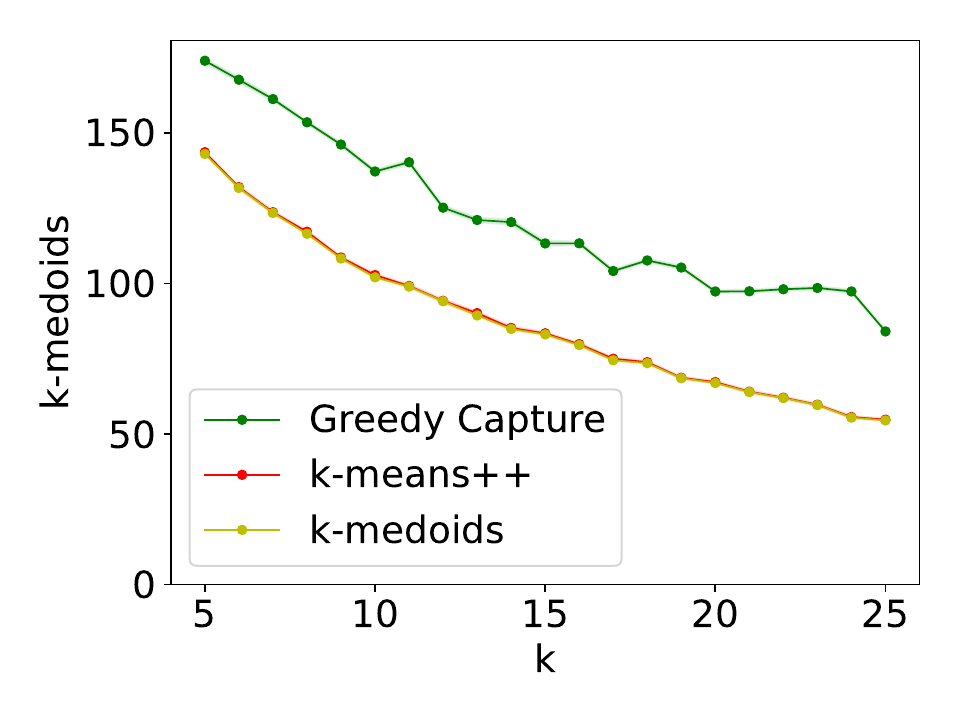}
    \end{subfigure}

    \caption{Census Income Dataset}
    \label{fig:adult-aver-core-fjr}
\end{figure}

\paragraph{Results.} In~\Cref{fig:adult-aver-core-fjr}, we see the results for the Census Income dataset; the results for $k$-means and $k$-medoids objectives for this dataset, along with results for the other two datasets, are relegated to \Cref{app:more-experiments} due to being qualitatively similar to the results presented here. 
According to all four fairness metrics, \gc is significantly fairer than both $k$-means++ and $k$-medoids, consistently across different values of $k$. Notably, the FJR approximation of \gc empirically stays very close to $1$ in all cases, in contrast to the higher worst-case bounds (\Cref{cor:fjr-gc}). Remarkably, the significant fairness advantage of \gc comes at a modest cost in accuracy: all three objective values (average within-cluster distance, $k$-means, and $k$-medoids) achieved by \gc are less than \emph{twice} those of $k$-means++ and $k$-medoids, across all values of $k$ and all three datasets! 

Lastly, our results are in contrast to the experiments of~\citet{chen2019proportionally} for centroid clustering, where \gc provides a worse core approximation than $k$-means++ on Iris and Diabetes datasets; as demonstrated in~\Cref{app:more-experiments}, this is not the case in non-centroid clustering. 

\section{Discussion}\label{sec:discussion}
We have initiated the study of proportional fairness in non-centroid clustering. Throughout the paper, we highlight several intriguing open questions. Probably the most important of these are whether we can achieve a better approximation than $O(n/k)$ of the core for the average loss, and whether the core is always non-empty for the maximum loss. In an effort to answer the latter question, in~\Cref{app:line-max} we show that the core is always non-empty for the maximum loss when the metric space is 1-dimensional (i.e., a line). This contrasts with the average loss, for which the core remains empty even on the line (see~\Cref{app:line-aver}). 

In our work, we have shown that there are remarkable differences between centroid and non-centroid clustering settings. One can consider a more general model, where the loss of an agent depends on both her cluster center and the other agents in her cluster. Investigating what proportional fairness guarantees can be achieved in this case is an exciting direction. Another intriguing question is whether we can choose the number of clusters $k$ intrinsically; this seems challenging as proportional fairness guarantees seem to depend on fixing $k$ in advance to define which coalitions can deviate. Lastly, while classical algorithms such as  $k$-means and $k$-centers are incompatible with the core and FJR in the worst case (see~\Cref{app:class-obje}), it is interesting to explore conditions under which they may be more compatible, and whether a fair clustering can be computed efficiently in such cases.

\section*{Acknowledgements}
Caragiannis was partially supported by the Independent Research Fund Denmark (DFF) under grant 2032-00185B. Shah was partially supported by an NSERC Discovery grant.

\bibliographystyle{unsrtnat}
\bibliography{abb,literature}

\newpage
\appendix
\section*{\centering Appendix}

\section{Bicriteria Approximation of the Core}\label{app:core-bicriteria}

Here, we consider a more general definition of the core.

\begin{definition}[$(\alpha,\delta)$-Core]
    For $\alpha \ge 1$, a $k$-clustering $C=(C_1,\ldots,C_k)$ is said to be in the $(\alpha,\delta)$-core if there is no group of agents $S \subseteq N$ with $|S| \ge \delta \cdot \sfrac{n}{k}$ such that $\alpha \cdot  \ell_i(S) < \ell_i(C(i))$ for all $i \in S$.   
\end{definition} 

\begin{theorem}
  \gc returns a clustering solution in the  $(\delta, \frac{2\delta}{\delta-1})$-core, for any $\delta>1$.
\end{theorem}
\begin{proof}
    Let $C=\{C_1,\ldots C_k\}$ be a solution that \gc returns. Suppose for contradiction that there exists $S\subseteq N$ with $|S|\geq \delta \cdot n/k$ such that 
    \begin{align*}
        \forall i \in S, \quad \quad \ell_{i} (C(i))>\frac{2\delta}{\delta-1} \cdot \ell_i(S).
    \end{align*}

Let $i^*$ be the agent that was the first among the agents in $S$ that was included in some cluster by the algorithm. Consider the time step before this happens and let $i'\in C(i^*)$ be the agent that had the minimum distance $R$ from the $\lceil \sfrac{n}{k}\rceil$-th agent in $C(i^*)$ among all agents that had not been included to clusters by the algorithm before. With similar arguments as in the proof of~\Cref{thm:core-ub}, we conclude that 
\begin{align*}
    \ell_{i^*}(C(i^*))\leq \left(2-\frac{3}{\lceil \sfrac{n}{k}\rceil}\right)\cdot R\leq 2\cdot  R.
\end{align*}
    
Again, with very similar arguments as in the proof of ~\Cref{thm:core-ub}, we can conclude that  that $S$ contains at least $|S|-\lceil \sfrac{n}{k}\rceil +1$ agents at distance at least $R$ from agent $i^*$. Thus,
\begin{align*}
    \ell_{i^*}(S) &=\frac{1}{|S|}\sum_{i\in S}{d(i^*,i)}\geq \frac{|S|-\lceil \sfrac{n}{k}\rceil+1}{|S|}\cdot R
    \geq \frac{|S|- \sfrac{n}{k}}{|S|}\cdot R
    \geq 
    \frac{\delta-1}{\delta} \cdot R.
\end{align*}
where the second inequality follows since $|S|\geq \delta \cdot  \sfrac{n}{k}$ and the theorem follows. 
\end{proof}

\section{Line}\label{app:line}

\subsection{Non-Emptiness of the Core for Maximum Loss}\label{app:line-max}

\begin{algorithm}	\caption{\sd}\label{alg:smallest-agent-ball-line}
%\SetAlgoNoLine
    \KwIn{ $N'\subseteq  N$, metric $d$, $k$, $t$}
    \KwOut{$S$}
 
    \eIf  {$|N'|<t$}{
    $S\gets N'$\;
    }{
       Label the agents from $1$ to $n'$, starting with the leftmost agent and moving to the right\;
       $d_{min}\gets d(1,n')$\;
       $i^*\gets 1$\;
    \For{$i=1$ to $n'-t$}{
        \If{$d(i,i+t)< d_{min}$}{
            $d_{min}\gets d(i,i+t)$\; 
            $i^*\gets i$; 
        }
        }
       }
        $S\gets \{i^*,\ldots, i^*+t\}$\;
\end{algorithm}

\begin{theorem}
    For the maximum loss in the line, \gcc(\sd) returns a solution in the core in $O(kn)$  time complexity.  
\end{theorem}
\begin{proof}
Let $C=\{C_1,\ldots, C_k\}$ be the solution that the algorithm returns.   
 Suppose for contradiction that there exists a group $S \subseteq N$, with $|S|\geq \sfrac{n}{k}$ such that $\ell_i(C(i))> \ell_{i}(S)$ for all $i \in S$. We denote the leftmost and rightmost agents in $S$ by 
$L$ and  $R$, respectively.
Let $i^*$ be  the first agent in $S$  that was assigned to some cluster. If we denote with $N'$ the set of agents that have not been disregarded before this happens, this means that $S\subseteq N'$. We denote the leftmost and rightmost agents in $C(i^*)$ by 
$L^*$ and  $R^*$, respectively.

Note that $i^*$ has incentives to deviate if and only if either $d(i^*,L)<d(i^*,L^*)$ or 
$d(i^*,R)<d(i^*,R^*)$, since otherwise $\ell_{i^*}(S)= \max\{d(i^*,L), d(i^*,R)\}\geq \ell_{i^*}(C(i^*))= \max\{d(i^*,L^*), d(i^*,R^*) \}  $. Without loss of generality, assume that $d(i^*,L)<d(i^*,L^*)$.  
Given the way that the algorithm operates, it is not hard to see that since $L^*$ and $i^*$ are included in $C(i^*)$ and $L$ is located between $L^*$ and $i^*$, then $L$ is also included in $C(i^*)$. Denote with $R'$ the $\ceil{\sfrac{n}{k}}$-th agent to the right of $L$ in $N'$. We notice that  the algorithm returns $C(i^*)$ instead of $S$, because $d(L^*,R^*)\leq d(L,R')\leq  d(L,R)$. But this means that
\[
\ell_L(C(i^L))=\ell_L(C(i^*))\leq d(L^*,R^*) \leq d(L,R)= \ell_L(S)
\]
and we reach in a contradiction. 

Since \gcc calls \sd at most $k$ times and \sd does at most $n$ iterations in each call, we easily see that the time complexity of \gcc(\sd) is $O(kn)$. 
\end{proof}

\subsection{Emptiness of the Core for Average Loss}\label{app:line-aver}

\begin{theorem}
For $k=2$ and the average loss, there exists an instance in the line where the core is empty.
\end{theorem}
\begin{proof}
Consider the instance with even $n>24$, where one agent, denoted by $a$, is located at position $0$, $n/2-1$ agents, denoted by the set $S_1$, are located at position $2$,  $n/2-1$ agents, denoted by the set $S_2$, are located at position $3$ and the last  agent denoted by $b$ is located at position $+\infty$.   

    Let  $C=(C_1,C_2)$ be any  clustering solution. Without loss of generality, suppose that $C_1$ contains $b$.  This means that all the agents that are part of $C_1$ have loss equal to infinity. Note that if $|C_2|\leq n/2-1$, then $|C_1\cap (S_1\cup S_2 \cup \{a\} ) |\geq  n/2+1$ which means that $n/k$ agents from  $S_1\cup S_2 \cup \{a\}$ could reduce their loss arbitrary much by deviating to their own cluster. Hence, $|C_2|\geq n/2$. Next, we distinguish to two cases:
    
\paragraph{Case I: $|C_2\cap S_2|\geq n/4$. }
In this case, for each agent $i$ in $S_1$, we have that
$\ell_i(C(i))\geq \frac{n/4}{n-1}\geq 1/4$. Moreover, note that $a$ always prefers to be in a cluster that consists by agents in $S_1$. Therefore, we have that  if $a$ and $S_1$ deviate to their own cluster, then  for each $i\in S_1$ $\ell_i(S_1\cup \{a\})= \frac{1}{n/2}<1/4$, where the last inequality follows from the fact that $n>8$.

\paragraph{Case II: $|C_2\cap S_2|< n/4$. }

Since $|C_2|\geq n/2$, we have that $|C_2\cap S_1|\geq n/4$. In this case, for each agent $i$ in $S_2$, we have that $\ell_i(C(i))\geq \frac{n/4}{n-1}\geq 1/4$.

Now, we distinguish to  two further subcases. First, suppose that $a$ belongs to $C_1$. Then, if the agents in $S_2$ and $a$  deviate to their own cluster,
 we have that  for each $i\in S_2$,   $\ell_i(S_2\cup \{a\})= \frac{3}{n/2}<1/4$, where the last inequality follows from the fact that $n>24$.  If $a$ is assigned to $C_2$, then all the agents in $S_1$ have incentives to deviate with an agent from $S_2$ that is assigned to $C_1$.

\end{proof}

\section{More Experimental Results}\label{app:more-experiments}
Here, we present additional experimental results on the Diabetes and Iris datasets.

\begin{figure}[htb!]
    \centering
    \begin{subfigure}[b]{0.45\textwidth}
    \caption{Core violation, average loss}
        \includegraphics[width=\textwidth]{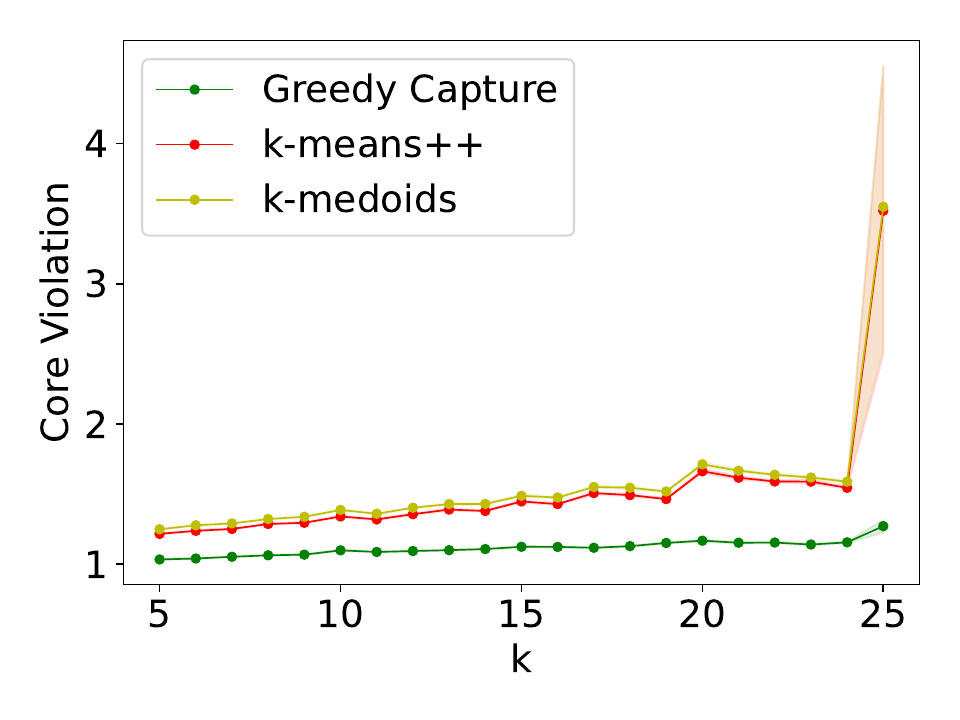}
    \end{subfigure}\hspace{0.02\textwidth}%
   \begin{subfigure}[b]{0.45\textwidth}
    \caption{FJR violation,  average loss}
        \includegraphics[width=\textwidth]{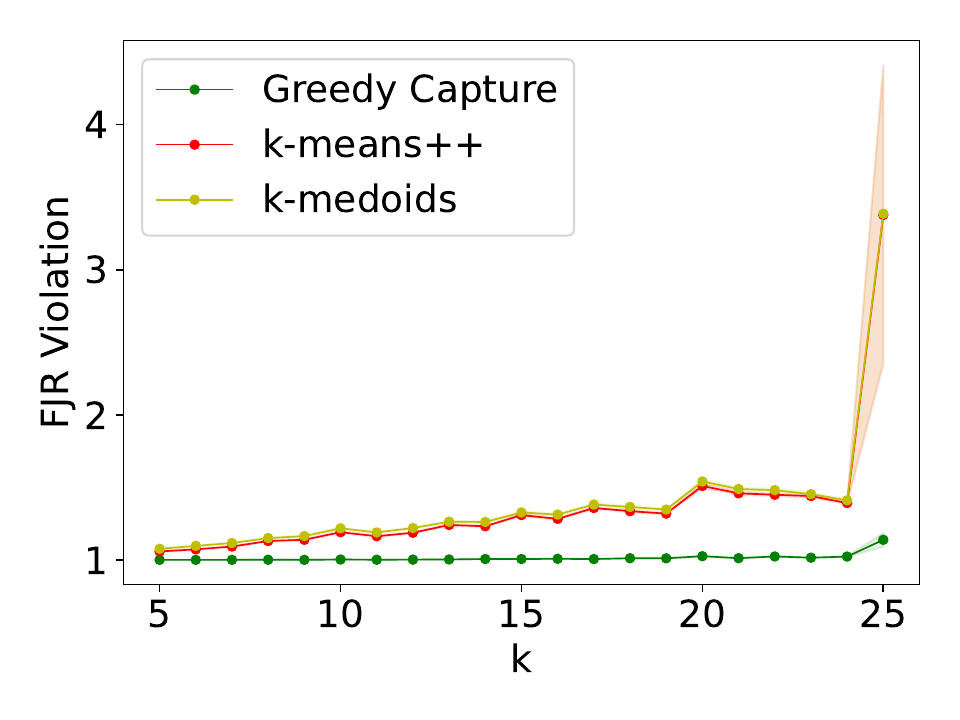}
    \end{subfigure}
    
    \begin{subfigure}[b]{0.45\textwidth}
     \caption{Core violation,  maximum loss}
        \includegraphics[width=\textwidth]{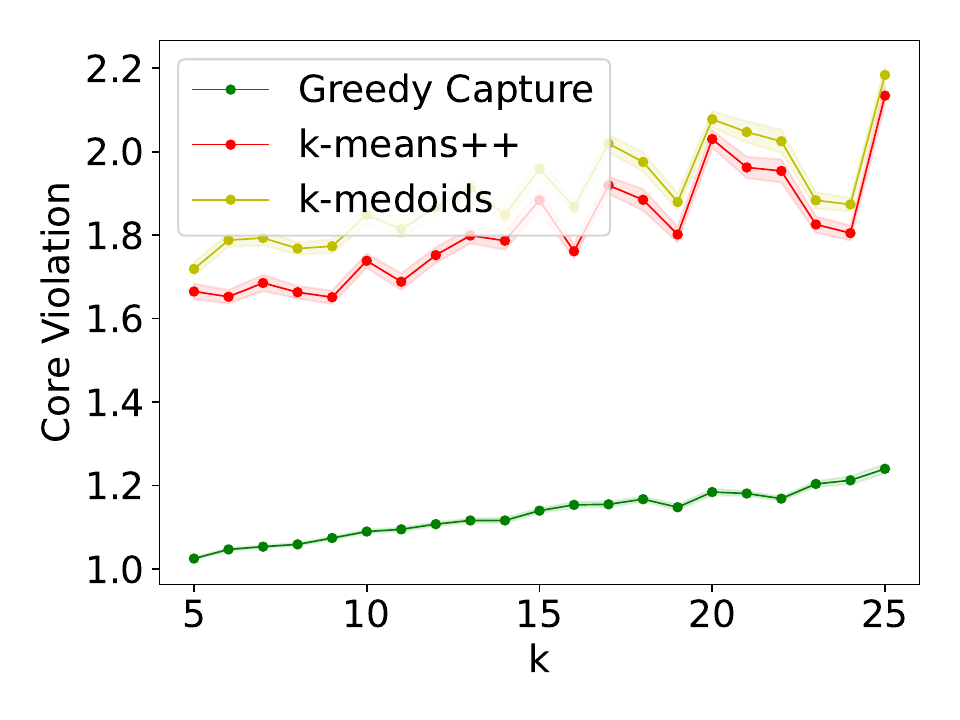}
    \end{subfigure}\hspace{0.02\textwidth}%   
  \begin{subfigure}[b]{0.45\textwidth}
       \caption{FJR violation,  maximum loss}
        \includegraphics[width=\textwidth]{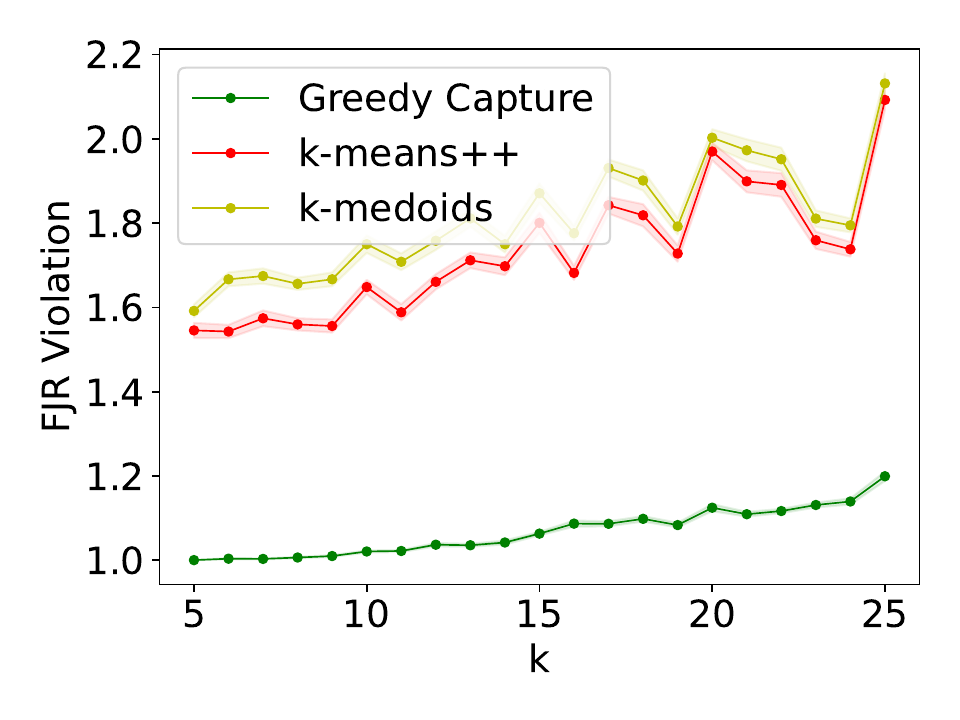}
    \end{subfigure}
    
    \begin{subfigure}[b]{0.31\textwidth}
     \caption{Avg within-cluster distance}
        \includegraphics[width=\textwidth]{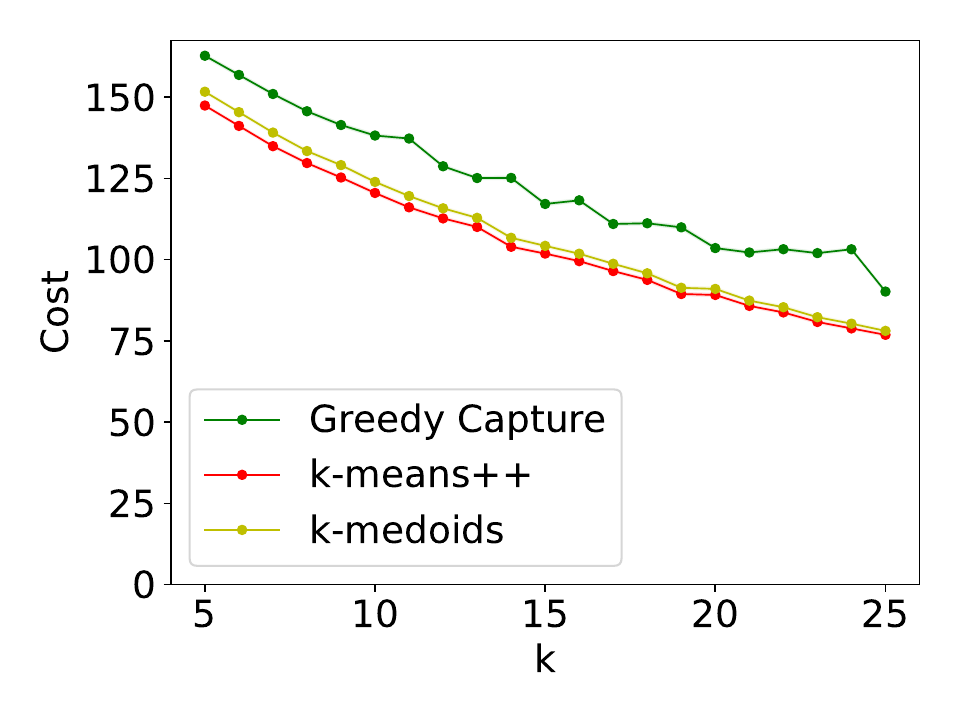}
    \end{subfigure}\hspace{0.02\textwidth}% 
    \begin{subfigure}[b]{0.31\textwidth}
    \caption{$k$-means objective}
        \includegraphics[width=\textwidth]{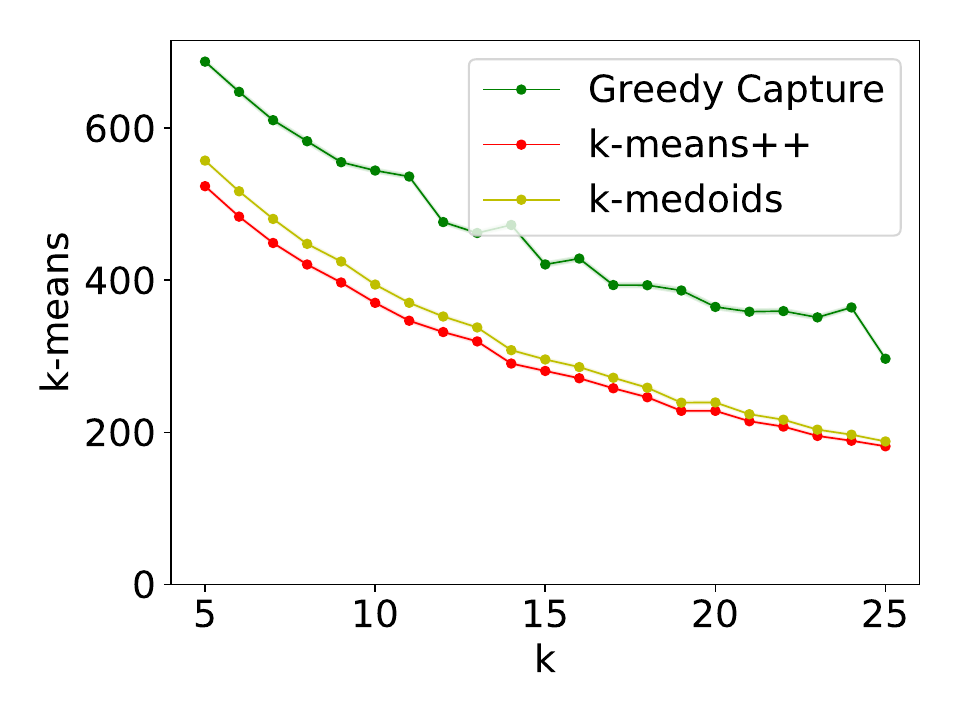}
    \end{subfigure}\hspace{0.02\textwidth}%
    \begin{subfigure}[b]{0.31\textwidth}
     \caption{$k$-medoids objective}
        \includegraphics[width=\textwidth]{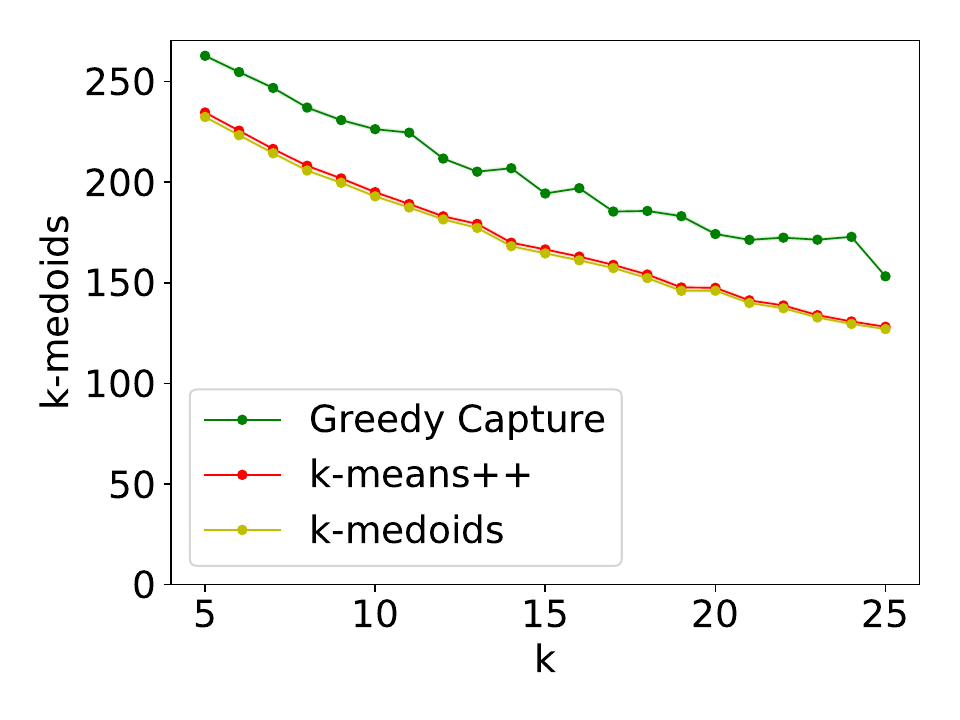}
    \end{subfigure}
   \caption{Diabetes dataset}
\end{figure}

\begin{figure}[htb!]
    \centering
    \begin{subfigure}[b]{0.45\textwidth}
    \caption{Core violation,  average loss}
        \includegraphics[width=\textwidth]{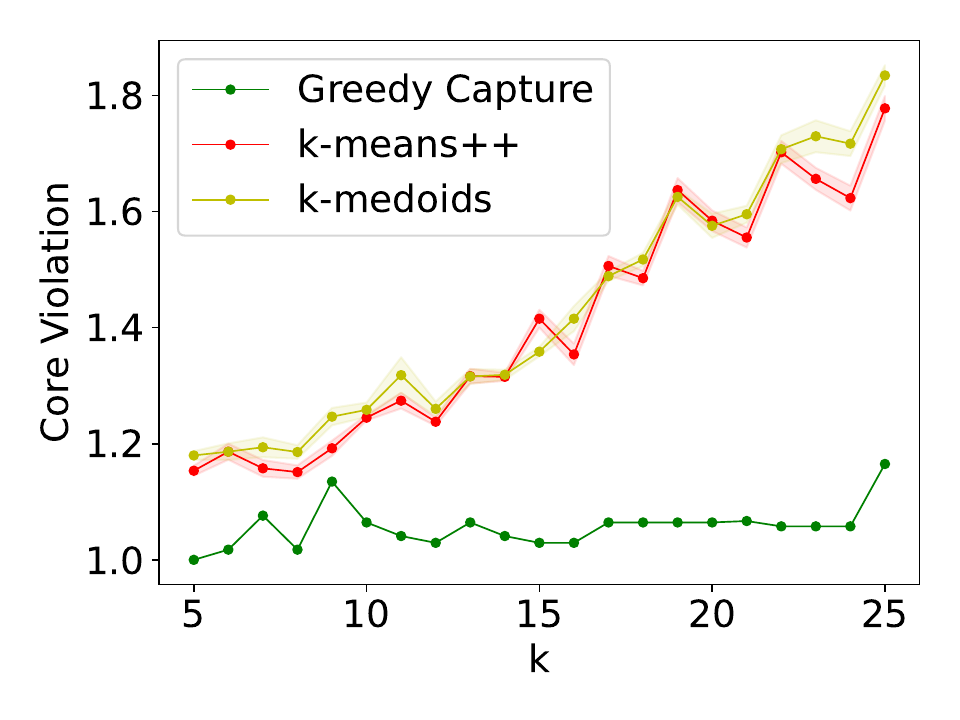}
    \end{subfigure}\hspace{0.02\textwidth}%
   \begin{subfigure}[b]{0.45\textwidth}
        \caption{FJR violation,  average loss}
        \includegraphics[width=\textwidth]{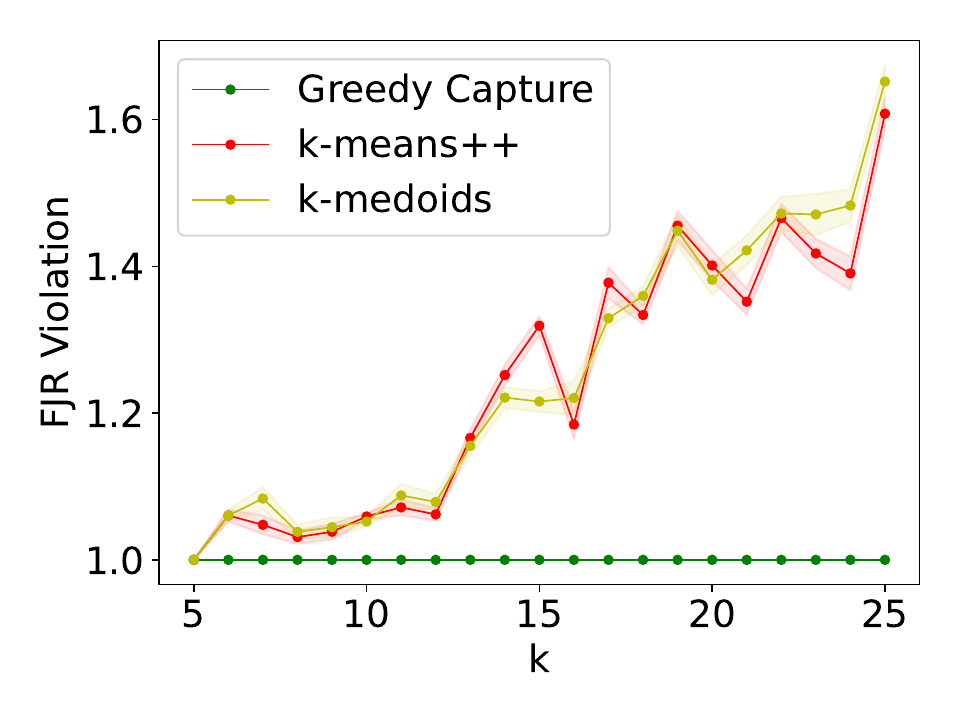}
    \end{subfigure}
    
    \begin{subfigure}[b]{0.45\textwidth}
        \caption{Core violation,  maximum loss}
        \includegraphics[width=\textwidth]{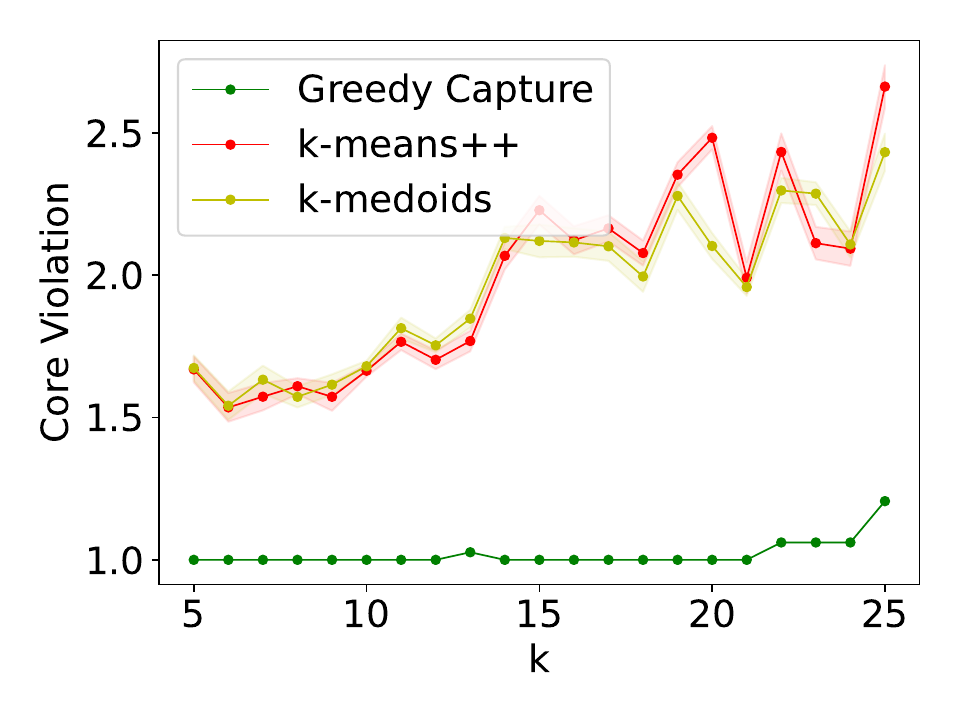}
    \end{subfigure}\hspace{0.02\textwidth}% 
  \begin{subfigure}[b]{0.45\textwidth}
  \caption{FJR violation, maximum loss}
        \includegraphics[width=\textwidth]{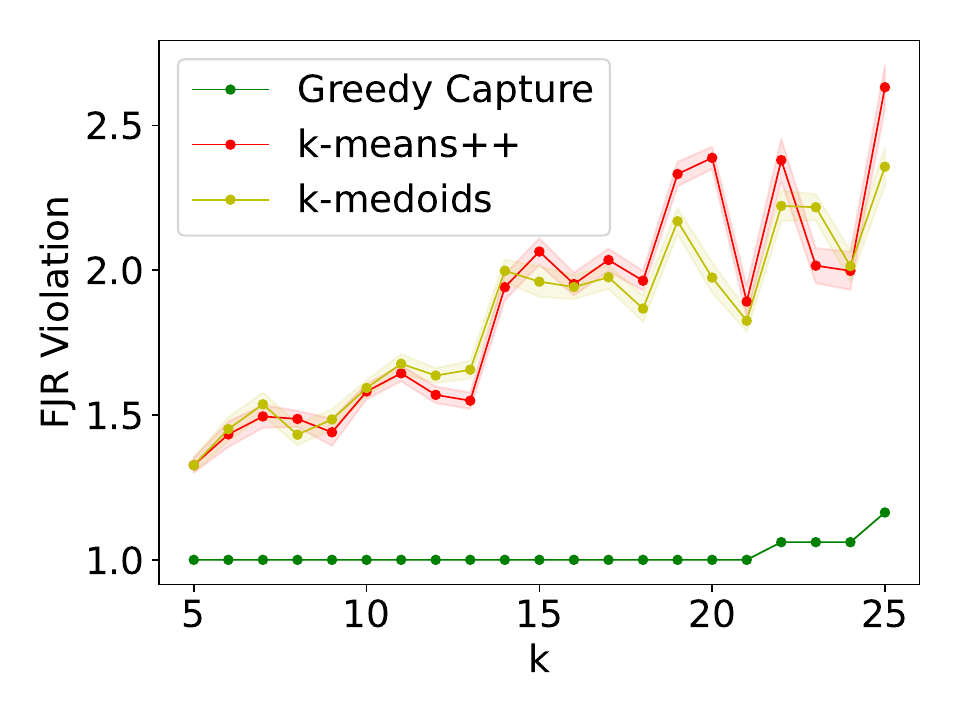}
    \end{subfigure}
    
    \begin{subfigure}[b]{0.31\textwidth}
       \caption{Avg within-cluster distance}
        \includegraphics[width=\textwidth]{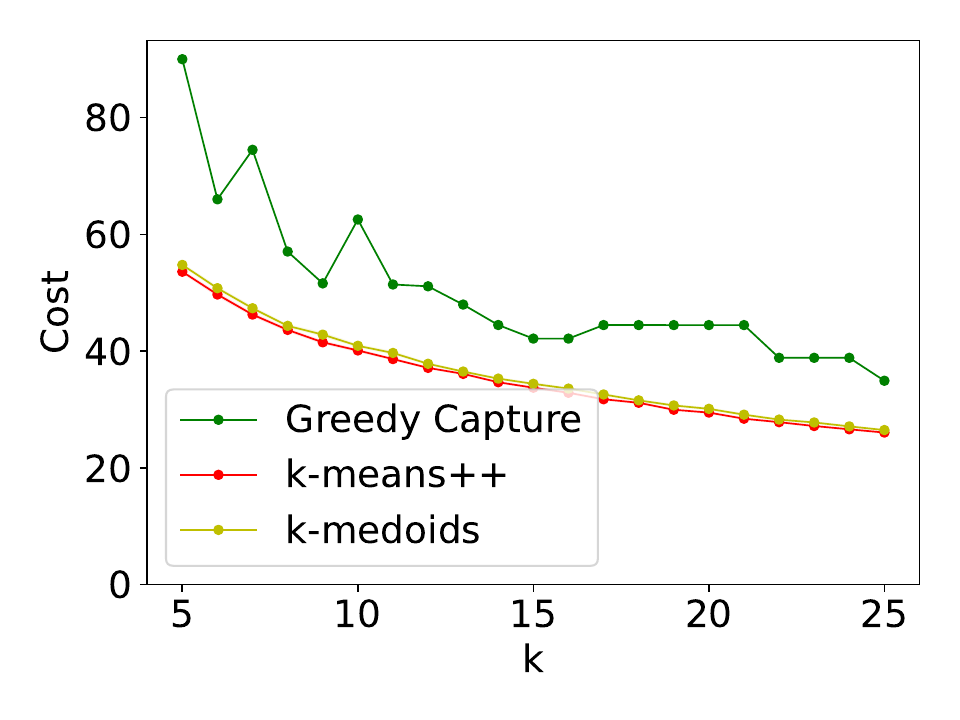}
    \end{subfigure}\hspace{0.02\textwidth}%
    \begin{subfigure}[b]{0.31\textwidth}
        \caption{$k$-means objective}
        \includegraphics[width=\textwidth]{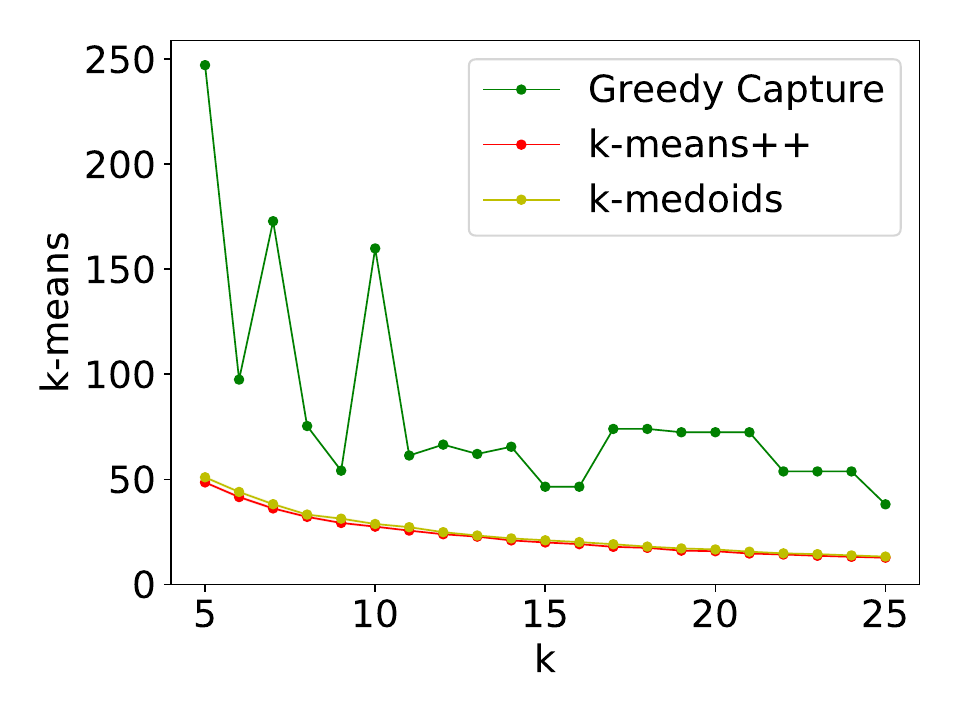}
    \end{subfigure}\hspace{0.02\textwidth}%
    \begin{subfigure}[b]{0.31\textwidth}
           \caption{$k$-medoids objective}
        \includegraphics[width=\textwidth]{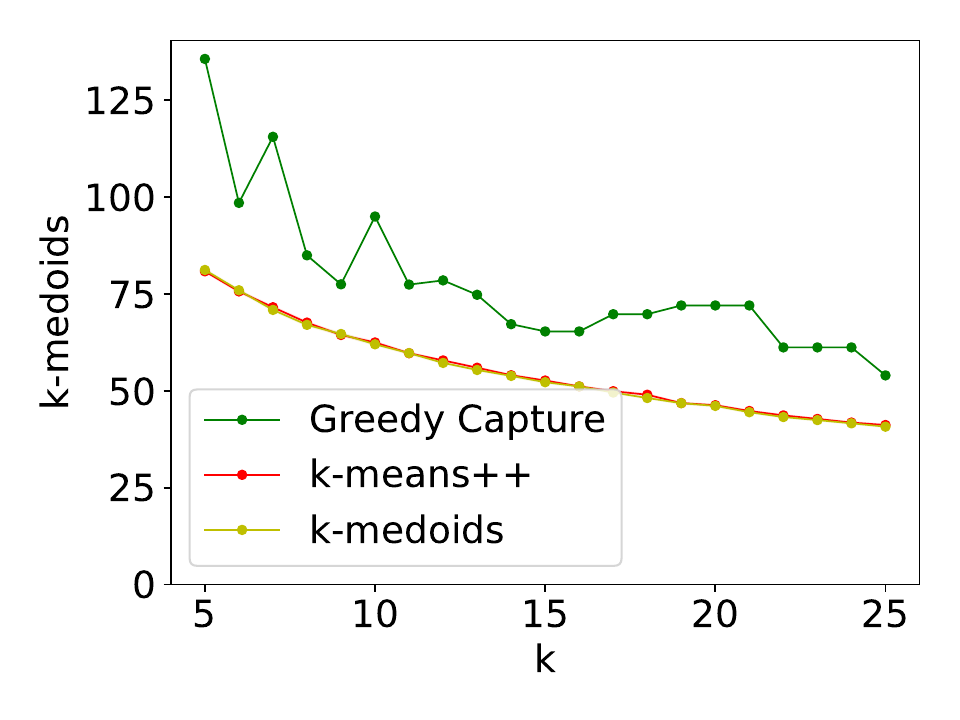}
    \end{subfigure}
  \caption{Iris dataset}
\end{figure}

\section{Incompatibility of FJR and Core with Classical Objectives }\label{app:class-obje}
Consider Example 1 from \citet{chen2019proportionally}. Classic algorithms such as $k$-center, $k$-means++, and $k$-median would   cluster all points at positions $a$ and $b$ together.  But if the points at $a$ deviate by forming a cluster, each of them improves from infinite loss to a finite loss. Therefore, these algorithms do not provide a finite approximation to the core or FJR. 

\end{document}